\documentclass{article} 
\usepackage{iclr2027_conference,times}

\ifdefined\pdfsuppressptexinfo \pdfsuppressptexinfo=-1 \fi
\ifdefined\pdfinfoomitdate \pdfinfoomitdate=1 \fi

\usepackage[utf8]{inputenc}
\usepackage[T1]{fontenc}
\usepackage{microtype}

\usepackage{amsmath,amssymb,amsthm,amsfonts}
\usepackage{mathtools}
\usepackage{bm}
\usepackage{nicefrac}

\usepackage{graphicx}
\usepackage{algorithm}
\usepackage{algpseudocode}

\usepackage{booktabs}
\usepackage{array}
\usepackage{multirow}
\usepackage{arydshln}
\usepackage{wrapfig}
\usepackage{flafter}
\usepackage{placeins}
\newcommand{\tabletextsize}{\footnotesize}

\usepackage[dvipsnames]{xcolor}

\DeclareMathOperator*{\argmax}{arg\,max}
\DeclareMathOperator*{\argmin}{arg\,min}

\theoremstyle{plain}
\newtheorem{theorem}{Theorem}[section]

\theoremstyle{definition}
\newtheorem{definition}[theorem]{Definition}

\theoremstyle{remark}

\usepackage{hyperref}
\usepackage{url}

\title{Multilingual Fine-Tuning via\\Localized Gradient Conflict Resolution}

\author{
Long P. Hoang\textsuperscript{1} \quad
Yiran Zhao\textsuperscript{2} \quad
Wei Lu\textsuperscript{3} \quad
Wenxuan Zhang\textsuperscript{1} \\
\textsuperscript{1}Singapore University of Technology and Design \\ \textsuperscript{2}Salesforce AI Research \quad
\textsuperscript{3}Nanyang Technological University \\
\texttt{long\_hoang@mymail.sutd.edu.sg} \quad \texttt{yiran.zhao@salesforce.com} \\ \texttt{wei.lu@ntu.edu.sg} \quad \texttt{wxzhang@sutd.edu.sg}
}

\makeatletter
\def\adl@drawiv#1#2#3{%
        \hskip.5\tabcolsep
        \xleaders#3{#2.5\@tempdimb #1{1}#2.5\@tempdimb}%
                #2\z@ plus1fil minus1fil\relax
        \hskip.5\tabcolsep}
\newcommand{\cdashlinelr}[1]{%
  \noalign{\vskip\aboverulesep
           \global\let\@dashdrawstore\adl@draw
           \global\let\adl@draw\adl@drawiv}
  \cdashline{#1}
  \noalign{\global\let\adl@draw\@dashdrawstore
           \vskip\belowrulesep}}
\makeatother

\iclrfinalcopy 
\begin{document}

\maketitle
\lhead{Preprint}

\begin{abstract}
    Cross-lingual versatility has become a defining capability of Large Language Models (LLMs). However, fine-tuning these models often induces \textit{negative interference}, where gains in one language come at another's expense. To address this, we cast multilingual fine-tuning as a multi-objective optimization (MOO) problem. Specifically, we introduce \textbf{Bucket-Level MOO}, a scalable distributed framework that applies gradient-based MOO algorithms locally, within the parameter buckets of distributed training. This enables conflict-aware updates without ever materializing a full gradient vector for any objective, and improves multilingual performance over standard fine-tuning by up to $4.3$ points. Theoretically, we define \textit{bucketwise Pareto stationarity}, show it is necessary for Pareto optimality under any fixed partition, and prove descent and convergence for representative Bucket-Level variants. Our work brings conventional MOO within reach for multilingual LLMs at a far lower cost. More broadly, Bucket-Level MOO is agnostic to both the objectives and the update rule, readily extending to other multi-task learning at scale.
\end{abstract}

\section{Introduction}

The multilingual capabilities of Large Language Models (LLMs) have advanced rapidly, driven largely by data-centric strategies such as translating English corpora or distilling knowledge from high-resource languages \citep{muennighoff-etal-2023-crosslingual, pan-etal-2024-g, chen-etal-2024-monolingual}. While these approaches improve average cross-lingual performance, they implicitly assume that a single set of shared parameters can accommodate diverse and often incompatible linguistic patterns. In practice, this assumption frequently breaks down during fine-tuning: updates that benefit one language can adversely affect others, leading to \textit{negative interference}, a phenomenon where optimizing for one language degrades performance on another \citep{conneau-etal-2020-unsupervised, wang-etal-2020-negative}. As multilingual training scales to more languages and tasks, mitigating such interference becomes a central challenge for achieving robust and balanced cross-lingual performance.

\begin{figure}[t]
    \centering
    \begin{minipage}{0.57\textwidth}
        \includegraphics[width=\linewidth]{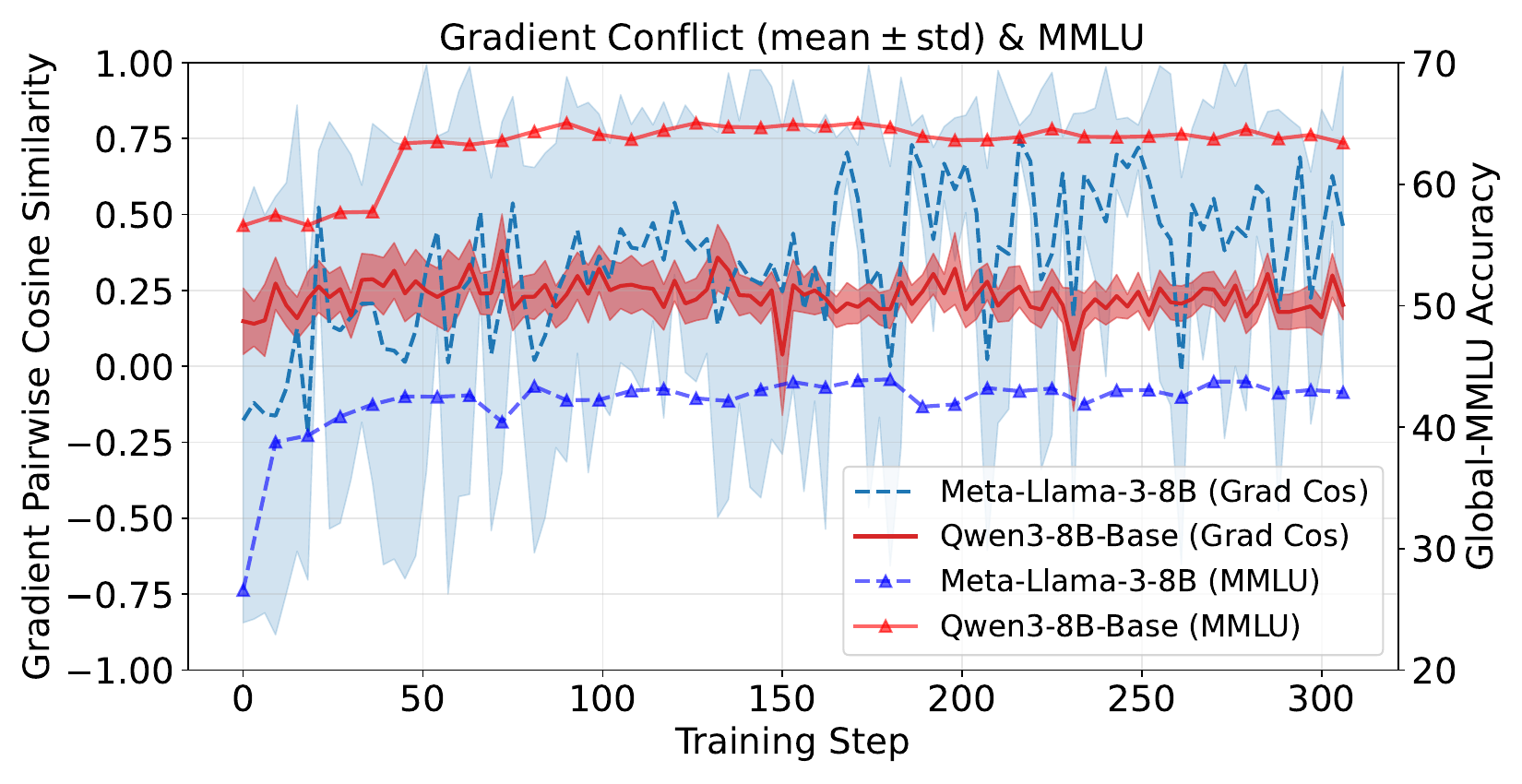}
    \end{minipage}\hfill
    \begin{minipage}{0.38\textwidth}
        \includegraphics[width=\linewidth]{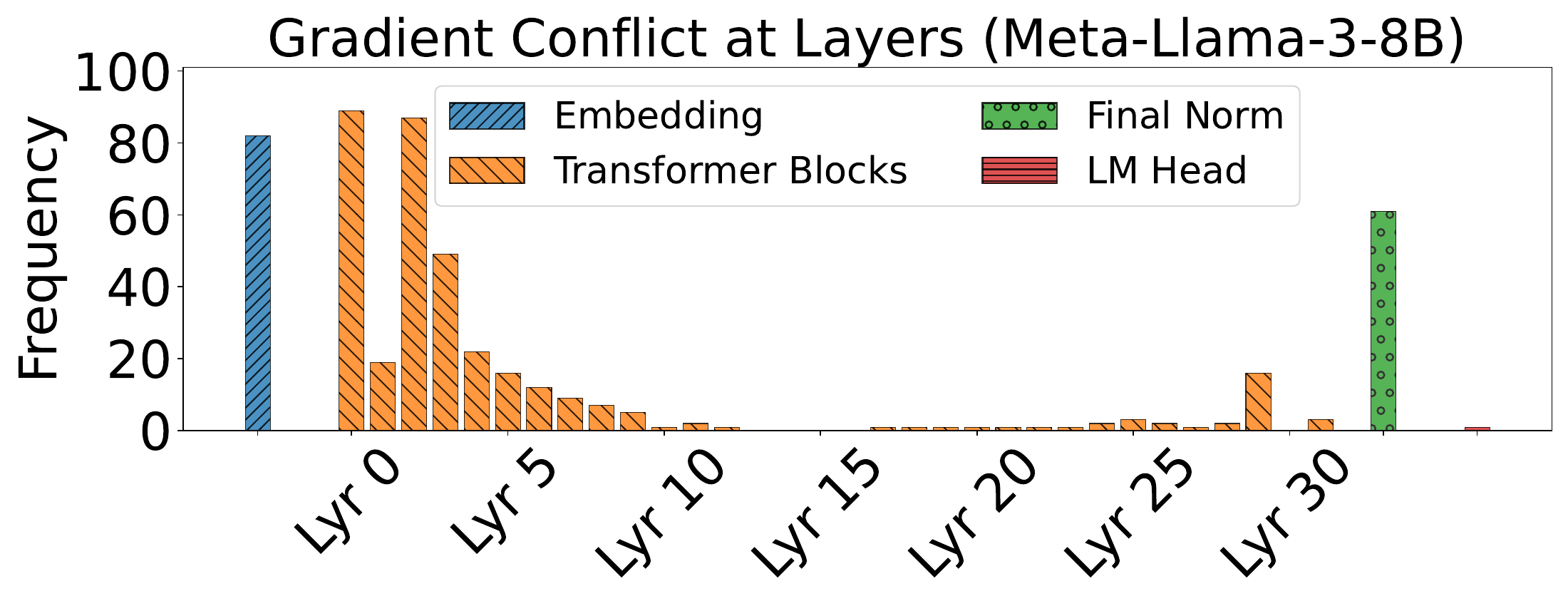}\\[0.2cm] 
        \includegraphics[width=\linewidth]{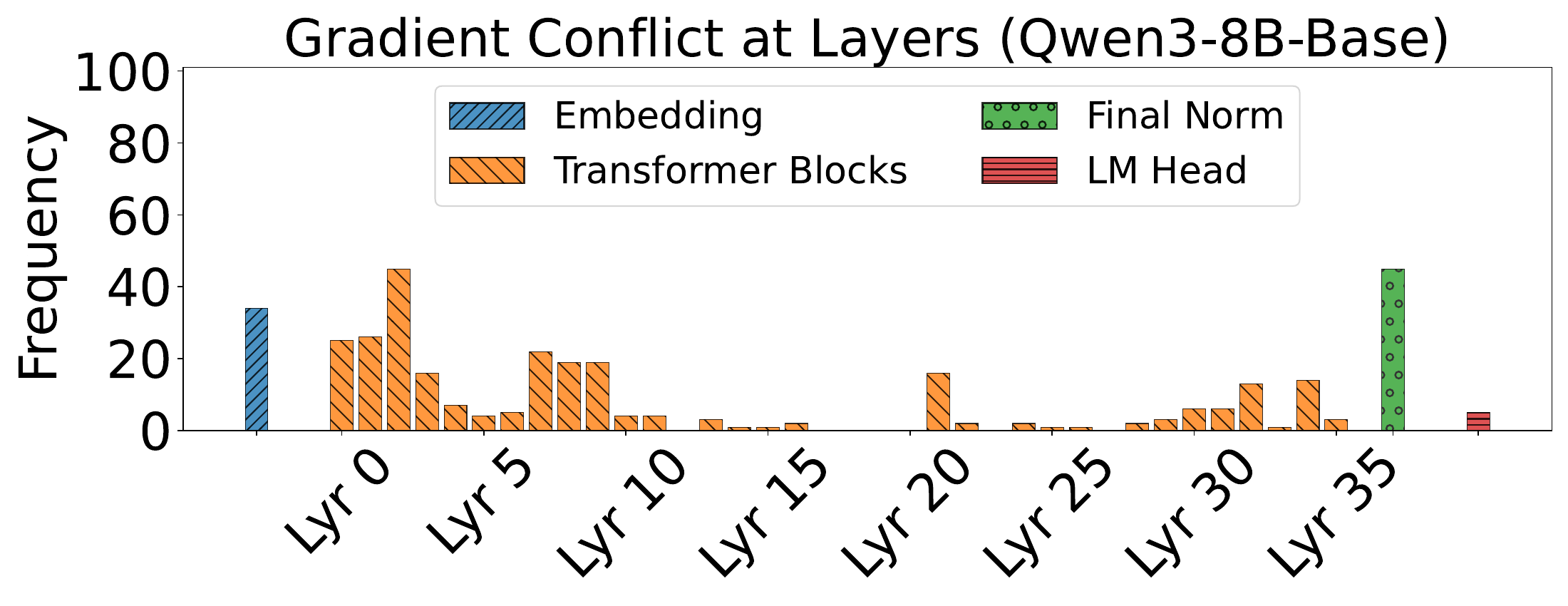}
    \end{minipage}
    \vspace{-0.15cm}
    \caption{Gradient conflict during multilingual fine-tuning with two base models on eight languages (\texttt{en, zh, it, ar, ko, id, bn, sw}). \textbf{Left:} Mean $\pm$ standard deviation of pairwise gradient cosine similarity across all language pairs, with accuracy scores on Global-MMLU dataset. \textbf{Right:} Per-layer frequency of conflicts, defined as the fraction of training steps where at least one language pair exhibits a negative gradient cosine similarity.
    }
    \label{fig:conflict}
    \vspace{-0.25cm}
 \end{figure}

From an optimization perspective, negative interference is fundamentally driven by conflicting gradients across language-specific objectives \citep{NeurIPS2018_Sener_Koltun, yu2020gradient, shi2023recon}. As illustrated in Figure \ref{fig:conflict} (Left), models with limited initial multilingual alignment (e.g., \texttt{Meta-Llama-3-8B}) exhibit severe, sustained gradient volatility throughout fine-tuning. Their pairwise cosine similarities frequently plunge into negative regimes with massive variance, severely constraining downstream progress. Conversely, stronger base multilingual models (e.g., \texttt{Qwen3-8B-Base}) are significantly more stable and maintain a largely positive average alignment. However, they still experience occasional, sharp gradient conflicts. These observations suggest the inherent instability of optimizing for diverse languages simultaneously.

A natural approach to mitigating such conflicts is to cast multilingual fine-tuning as a Multi-Objective Optimization (MOO) problem \citep{wang2021gradient, mao2022lessforgetting}. In this framework, each language defines a separate objective, and the ultimate goal is to achieve \textit{Pareto optimality}--a state where no language's performance can be improved without degrading another's. A practical MOO method therefore seeks an update that improves all language objectives when possible.

Conventional gradient-based MOO methods, such as MGDA \citep{NeurIPS2018_Sener_Koltun}, CAGrad \citep{liu2021conflict}, and PCGrad \citep{yu2020gradient}, offer principled ways to mitigate conflicting objectives and find possible common descent directions. However, these methods require access to full gradient vectors across all objectives, incurring substantial memory overhead in distributed training settings. Beyond scalability, they also treat the model as a monolithic parameter space, implicitly assuming conflicts are globally uniform. In practice, linguistic processing within LLMs is structurally heterogeneous, e.g., various layers perform distinct roles \citep{zhao2024how, bandarkar2026multilingual}. As shown in Figure \ref{fig:conflict} (Right), gradient conflicts are highly localized, often concentrating in specific layers (e.g., early or late layers). Resolving such conflicts where they arise is therefore a necessity rather than a preference at LLM scale: it matches the structure of the conflicts themselves and avoids assembling a full gradient for every language in order to act on them.

We therefore introduce \textit{Bucket-Level MOO}, a scalable framework that turns the parameter partitions already used by distributed training into buckets of multi-objective optimization. Rather than introducing a new MOO solver, it provides a protocol that allows a broad class of MOO and multi-task learning methods to operate locally without reconstructing full-model objective gradients. This partitioned formulation also yields a natural first-order characterization: we introduce \textit{bucketwise Pareto stationarity}, prove that it is necessary for Pareto optimality under any fixed parameter partitions, and establish descent and convergence guarantees for the MGDA and CAGrad variants.

Our main contributions are summarized as follows:
\begin{itemize}
   \item \textit{Practically,} we characterize the layer-wise heterogeneity of negative interference in LLMs, motivating localized conflict resolution as an efficient, effective alternative to the global one.
   \item \textit{Methodologically,} we propose \textit{Bucket-Level MOO}, a solver-agnostic framework that integrates into distributed training, enabling conflict resolution at any granularity while bounding the memory it adds by a single bucket.
   \item \textit{Theoretically,} we introduce \textit{bucketwise Pareto stationarity} for arbitrary fixed parameter partitions as a necessary condition for Pareto optimality. We further establish descent and convergence guarantees for the MGDA and CAGrad variants of Bucket-Level MOO.
   \item \textit{Empirically}, we demonstrate through extensive experiments that Bucket-Level MOO significantly improves both in-distribution and out-of-distribution multilingual capabilities.
\end{itemize}

\section{Methodology}
\label{sec:BK-MOO}

\subsection{Multilingual MOO and Localized Gradient Conflicts}
\label{sec:moo_defs}
\label{sec:motivation}

Given target languages \(\mathcal{T}:=\{1,2,\dots,T\}\), we formulate multilingual fine-tuning of a model parameterized by \(\bm{\theta}\) as the multi-objective problem
\begin{align}
    \min_{\bm{\theta}}\; \mathbf{L}(\bm{\theta}), \quad \text{where} \quad \mathbf{L}(\bm{\theta}) = \left[ \mathcal{L}_1(\bm{\theta}), \dots, \mathcal{L}_T(\bm{\theta})\right]^\top,
    \label{equ:main}
\end{align}
where \(\mathcal{L}_t\) is the loss for language \(t\) on its dataset \(\mathcal{D}_t\). Two language gradients conflict when \(\langle\nabla_{\bm{\theta}}\mathcal{L}_i,\nabla_{\bm{\theta}}\mathcal{L}_j\rangle<0\). Indeed, a first-order approximation of an update for language \(i\) gives
\begin{align}
    \Delta\mathcal{L}_j
    \approx -\eta\left\langle\nabla_{\bm{\theta}}\mathcal{L}_j,
    \nabla_{\bm{\theta}}\mathcal{L}_i\right\rangle,
\end{align}
so a conflicting update increases \(\mathcal{L}_j\) and causes negative interference.

\textbf{Conventional Gradient-Based MOO} is therefore used to construct an update that explicitly accounts for the competing language objectives. \textbf{MGDA} \citep{fliege2000steepest,desideri2012multiple} finds a minimum-norm convex combination of objective gradients, \textbf{CAGrad} \citep{liu2021conflict} optimizes worst-case improvement near their average, and \textbf{PCGrad} \citep{yu2020gradient} projects away pairwise conflicts. Appendix~\ref{app:formal} presents a detailed overview of these methods and the relevant background.

These methods presume that a full gradient per objective is readily available, an assumption that modern LLM training breaks in two ways: conflicts concentrate in a few layers rather than spreading uniformly across the model (Figure~\ref{fig:conflict}), and model states and gradients are sharded across devices. This leads to two practical weaknesses. \textit{Memory Cost:} holding a full gradient per objective together with the solver workspace is prohibitively expensive at LLM scale. \textit{Deferred reduction:} global MOO cannot reduce any part of the gradient until the full gradient of every objective is available, so each GPU must keep its entire unreduced gradient until the backward pass ends.

Distributed training \citep{rajbhandari2020zeromemoryoptimizationstraining,zhao2023pytorchfsdpexperiencesscaling} already partitions parameters and gradients into fixed-size units, called \textit{communication buckets}, so that a large model can be processed incrementally. We reuse this existing structure for conflict resolution and propose \textbf{Bucket-Level MOO}: as soon as the gradients of a bucket become available, the selected MOO rule is applied to that bucket and the resolved update is passed to the usual cross-device communication step before training moves to the next one (Figure~\ref{fig:BK-MOO}). Conflicts are thus resolved at the granularity at which they arise, no complete language gradient is ever materialized, and the additional memory is bounded by the size of a single bucket, independent of the number of objectives \(T\).

\subsection{Bucket-Level Multi-Objective Optimization}
\label{sec:bucket_moo}

Let \(\mathcal{P}=\{\bm{\theta}_{(1)},\dots,\bm{\theta}_{(K)}\}\) be a fixed partition of the model parameters into \(K\) disjoint buckets. Bucket-Level MOO applies a conflict-resolution rule independently within each bucket, so that the granularity of optimization matches both the locality of the conflicts and the communication structure of distributed training.
\begin{figure}[t!]
    \centering
    \includegraphics[width=\linewidth]{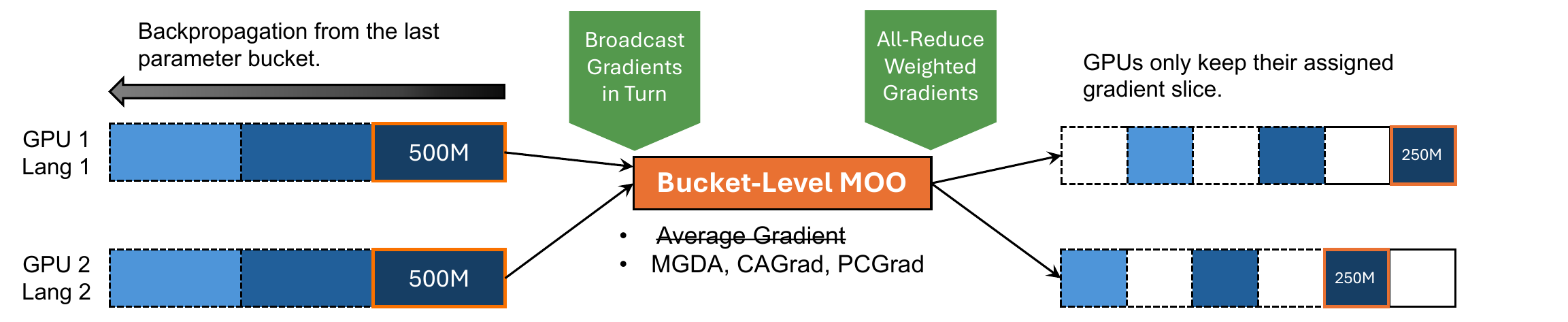}
    \caption{Bucket-Level MOO for two languages. Backpropagation produces gradients one fixed-size parameter bucket at a time. For each bucket, the GPUs broadcast their language gradients in turn through one reusable buffer, an MOO solver turns their inner products into combination weights, and the weighted gradients are all-reduced, with each GPU keeping only its assigned slice.}
    \label{fig:BK-MOO}
\end{figure}

For clarity, suppose \(T\) GPUs each process a micro-batch from a distinct language, so that GPU \(t\) computes \(\mathcal{L}_t\). Backpropagation makes \(\nabla_{\bm{\theta}_{(k)}}\mathcal{L}_t\) available on GPU \(t\) as soon as bucket \(k\) is complete, and Bucket-Level MOO exchanges only these bucket-local gradients:
\begin{align*}
    \mathbf{G}_{(k)} = \left\{ \nabla_{\bm{\theta}_{(k)}}\mathcal{L}_1, \dots, \nabla_{\bm{\theta}_{(k)}}\mathcal{L}_T \right\}.
\end{align*}
The selected MOO rule maps \(\mathbf{G}_{(k)}\) to a single update for bucket \(k\), which is immediately reduced and distributed across the GPUs before training proceeds to the next bucket. A complete model gradient is therefore never formed for any language.

We instantiate the per-bucket rule with the three solvers introduced in Section~\ref{sec:motivation}. The equations below make explicit that each acts only on the current bucket.

\paragraph{Bucket-Level MGDA.}
For each bucket \(k\) we solve the local quadratic program
\begin{align}
    \min_{\bm{\alpha}^{(k)} \in \Delta} \left\| \sum_{t=1}^T \alpha^{(k)}_t \nabla_{\bm{\theta}_{(k)}}\mathcal{L}_t \right\|_2^2, \qquad \Delta=\left\{\bm{\alpha}\geq\mathbf{0}:\sum_{t=1}^T\alpha_t=1\right\}.
    \label{eq:shard-mgda}
\end{align}
Its solution defines the update \(\mathbf{d}^{\text{MGDA}}_{(k)}=\sum_{t=1}^T\alpha^{(k)}_t\nabla_{\bm{\theta}_{(k)}}\mathcal{L}_t\), the minimum-norm element of the convex hull of the bucket gradients. The weights \(\bm{\alpha}^{(k)}\) are free to differ across buckets.

\paragraph{Bucket-Level CAGrad.}
Writing \(\mathbf{g}^{avg}_{(k)}=\frac{1}{T}\sum_{t=1}^T\nabla_{\bm{\theta}_{(k)}}\mathcal{L}_t\) for the local average gradient, we solve the CAGrad subproblem on bucket \(k\):
\begin{align}
    \mathbf{d}^{\text{CAGrad}}_{(k)} = \argmax_{\mathbf{d}_{(k)}} \min_{t} \langle \nabla_{\bm{\theta}_{(k)}}\mathcal{L}_t, \mathbf{d}_{(k)} \rangle \quad \text{s.t.} \quad \|\mathbf{d}_{(k)} - \mathbf{g}^{avg}_{(k)}\| \leq c \|\mathbf{g}^{avg}_{(k)}\|.
    \label{eq:bucket-cagrad}
\end{align}
The update therefore maximizes the worst-case alignment with the bucket gradients while staying within a ball around the local average, whose radius is set by \(c\in[0,1)\).

\paragraph{Bucket-Level PCGrad.}
Writing \(\mathbf{g}^{(k)}_t=\nabla_{\bm{\theta}_{(k)}}\mathcal{L}_t\), each conflicting pair is deconflicted using only inner products computed within the bucket:
\begin{align}
    \langle\mathbf{g}^{(k)}_i,\mathbf{g}^{(k)}_j\rangle<0
    \quad\Longrightarrow\quad
    \mathbf{g}^{(k)}_i\leftarrow
    \mathbf{g}^{(k)}_i-
    \frac{\langle\mathbf{g}^{(k)}_i,\mathbf{g}^{(k)}_j\rangle}
    {\|\mathbf{g}^{(k)}_j\|_2^2}\mathbf{g}^{(k)}_j.
\end{align}
After all conflicting pairs are projected, the bucket update averages the deconflicted gradients, \(\mathbf{d}^{\text{PCGrad}}_{(k)}=\frac{1}{T}\sum_{t=1}^T\mathbf{g}^{(k)}_t\), which keeps its scale comparable to the MGDA and CAGrad updates.

In every case, the bucket updates concatenate into \(\mathbf{d} = [\mathbf{d}_{(1)}, \dots, \mathbf{d}_{(K)}]\). Figure~\ref{fig:BK-MOO} traces one bucket from gradient arrival to reduced update, and Appendix~\ref{app:alg} gives pseudocode.

\subsection{Bucketwise Pareto Stationarity and Convergence}
\label{sec:refined_pareto}

Having defined the bucket-level variants, we next ask what condition should hold when no further
update is produced. The ultimate goal remains Pareto optimality, where no language can improve without degrading another. However, resolving conflicts bucket by bucket changes the first-order condition that characterizes a fixed point. We say that \(\bm{\theta}^*\) is \textit{bucketwise Pareto stationary} (BPS) if
\begin{equation}
\mathbf{0} \in \prod_{k=1}^{K} \operatorname{conv}\left\{\nabla_{\bm{\theta}_{(k)}}\mathcal{L}_1(\bm{\theta}^*),\dots,\nabla_{\bm{\theta}_{(k)}}\mathcal{L}_T(\bm{\theta}^*)\right\}.
\label{eq:bps}
\end{equation}
Equivalently, the language gradients within each bucket admit a vanishing convex combination, \(\sum_{t=1}^T\alpha_t^{(k)}\nabla_{\bm{\theta}_{(k)}}\mathcal{L}_t(\bm{\theta}^*)=\mathbf{0}\) for some \(\bm{\alpha}^{(k)}\in\Delta\).  Under any fixed grouping, every Pareto-optimal solution is BPS (Theorem~\ref{theorem:bps_necessary}); Bucket-Level MOO therefore never discards a solution that the conventional one would retain. The next two results show that the MGDA and CAGrad variants converge to such points.

\begin{theorem}
    Let \(\bm{\alpha}^{(k)}\in\Delta\) solve Equation~\ref{eq:shard-mgda}, let \(\mathbf{d}^{\mathrm{MGDA}}_{(k)}=\sum_{t=1}^T\alpha^{(k)}_t\nabla_{\bm{\theta}_{(k)}}\mathcal{L}_t\), and concatenate \(\mathbf{d}^{\mathrm{MGDA}}=[\mathbf{d}^{\mathrm{MGDA}}_{(1)},\dots,\mathbf{d}^{\mathrm{MGDA}}_{(K)}]\). Then:
    \begin{enumerate}
        \item[(a)] \(\mathbf{d}^{\mathrm{MGDA}}=\mathbf{0}\) if and only if the current parameters are BPS with respect to \(\mathcal{P}\).
        \item[(b)] If \(\mathbf{d}^{\mathrm{MGDA}}\neq\mathbf{0}\), then \(-\mathbf{d}^{\mathrm{MGDA}}\) is a common descent direction and
        \(\langle\nabla_{\bm{\theta}}\mathcal{L}_t,\mathbf{d}^{\mathrm{MGDA}}\rangle\geq\|\mathbf{d}^{\mathrm{MGDA}}\|_2^2>0\) for every \(t\).
        \item[(c)] Suppose each \(\mathcal{L}_t\) is lower bounded and has an \(L_t\)-Lipschitz gradient. For exact updates \(\bm{\theta}^{n+1}=\bm{\theta}^n-\eta\mathbf{d}^n\) with \(0<\eta<2/L\), where \(L=\max_tL_t\), we have \(\sum_{n=0}^{\infty}\|\mathbf{d}^n\|_2^2<\infty\) and hence \(\|\mathbf{d}^n\|_2\to0\). Every accumulation point is BPS.
    \end{enumerate}
    \label{theorem:shard-mgda}
\end{theorem}
BPS is therefore the exact convergence certificate of Bucket-Level MGDA: the update vanishes only at BPS points, and every nonzero update decreases all language losses simultaneously.

\begin{theorem}
    Let \(\mathbf{d}^{\mathrm{CAGrad}}_{(k)}\) solve the local CAGrad problem with \(0\leq c<1\), local average gradient \(\mathbf{g}^{avg}_{(k)}\), and global concatenation \(\mathbf{d}^{\mathrm{CAGrad}}\). Then:
    \begin{enumerate}
        \item[(a)] If \(\nabla\mathcal{L}_{avg}\neq\mathbf{0}\), then \(-\mathbf{d}^{\mathrm{CAGrad}}\) is a descent direction for \(\mathcal{L}_{avg}\), with \(\langle\nabla\mathcal{L}_{avg},\mathbf{d}^{\mathrm{CAGrad}}\rangle\geq\frac{1-c^2}{2}\|\nabla\mathcal{L}_{avg}\|_2^2>0\).
        \item[(b)] \(\mathbf{d}^{\mathrm{CAGrad}}=\mathbf{0}\) if and only if \(\nabla\mathcal{L}_{avg}=\mathbf{0}\). Thus a fixed point is standard Pareto stationary with the single uniform global weight \(\alpha_t=1/T\), and is consequently also BPS.
        \item[(c)] If \(\mathcal{L}_{avg}\) is lower bounded and has an \(L_F\)-Lipschitz gradient, then exact updates with \(0<\eta\leq1/L_F\) satisfy \(\sum_n\|\nabla\mathcal{L}_{avg}(\bm{\theta}^n)\|_2^2<\infty\); thus \(\|\nabla\mathcal{L}_{avg}(\bm{\theta}^n)\|_2\to0\) and \(\|\mathbf{d}^n\|_2\to0\).
    \end{enumerate}
    \label{theorem:shard-cagrad}
\end{theorem}
Bucket-Level CAGrad provides a guarantee of a different kind: the updates decrease the average language loss and drive its gradient to zero. The points at which the update vanishes are again BPS.

\section{Experiments}
\label{sec:exp}

\subsection{Experimental Setup}
\label{sec:exp_setup}

\paragraph{Models and Baselines.}
Four base LLMs are evaluated: \texttt{Meta-Llama-3-8B}, \texttt{Llama-3.1-8B}, \texttt{Qwen3-4B-Base}, and \texttt{Qwen3-8B-Base}. We compare the pretrained base model, Vanilla SFT, Vanilla SFT with a language-balanced data loader, and the three Bucket-Level MOO variants.

\paragraph{LIMA--s1 Training Data and Evaluation.}
We construct an English corpus comprising 1,000 LIMA conversational pairs \citep{zhou2023lima} and 630 s1k-1.1 reasoning samples \citep{muennighoff-etal-2025-s1} distilled from DeepSeek-R1.\footnote{We filter s1k-1.1 to include only samples with their short CoT traces marked as correctly.} We use \texttt{Gemini 2.0 Flash} to translate this corpus into eight languages (\texttt{en, zh, it, ar, ko, id, bn, sw}) spanning the high- to low-resource spectrum.
We evaluate three complementary competencies:
\begin{itemize}
    \item \textbf{Natural language understanding:} BELEBELE \citep{bandarkar-etal-2024-belebele} measures parallel reading comprehension.
    \item \textbf{Technical reasoning:} Multilingual ARC-E \citep{han2025mubenchassessmentmultilingualcapabilities} and the low-medium subsets of PolyMath \citep{wang2025polymath} assess scientific and mathematical reasoning.
    \item \textbf{Knowledge-based question answering:} Global-MMLU-Lite \citep{singh2024globalmmluunderstandingaddressing} measures world-knowledge retention and transfer.
\end{itemize}
For a consistent comparison across tasks, we use the 13 languages shared by all four benchmarks. We report results separately for the eight \textit{seen} training languages and five \textit{unseen} languages (\texttt{de, es, fr, ja, pt}), the latter measuring zero-shot cross-lingual generalization.

\paragraph{WMT24++ Training Data and Evaluation.}
For multilingual machine translation, we use the WMT24++ dataset \citep{deutsch-etal-2025-wmt24} and jointly fine-tune each backbone on four languages: Vietnamese (\texttt{vi}), Indonesian (\texttt{id}), Thai (\texttt{th}), and Tamil (\texttt{ta}). We evaluate the fine-tuned models on the corresponding FLORES-200 subsets included in SEA-HELM \citep{susanto-etal-2025-sea}. Following SEA-HELM, we measure translation quality using MetricX-wmt24, normalized to the $0$--$100$ scale, where higher values indicate better translation quality \citep{juraska-etal-2024-metricx}.

Fine-tuning details and computational costs are provided in Appendices \ref{appendix:exp} and \ref{app:time}, respectively.

\subsection{Multilingual Instruction Fine-Tuning on LIMA--s1}
\label{sec:main_results}

\begin{table*}[t!]
    \centering
    \caption{Main results (higher is better) after fine-tuning on the multilingual LIMA--s1 dataset. Per-language results in Appendix \ref{app:per_language_analysis} show consistent gains from Bucket-Level MOO across languages.}
    \label{tab:main_compact_aligned}
    \resizebox{\textwidth}{!}{
    \renewcommand{\arraystretch}{1.2}
    \setlength{\tabcolsep}{3pt}
    \begin{tabular}{l*{10}{>{\centering\arraybackslash}p{1.4cm}}}
    \toprule
    \multirow{2}{*}{\textbf{Method}} & \multicolumn{5}{c}{\textit{8 Seen Languages}} & \multicolumn{5}{c}{\textit{5 Unseen Languages}} \\
    \cmidrule(r){2-6} \cmidrule(r){7-11}
    & \textbf{BELE} & \textbf{ARC-E} & \textbf{PolyMath} & \textbf{MMLU} & \textbf{Avg.} & \textbf{BELE} & \textbf{ARC-E} & \textbf{PolyMath} & \textbf{MMLU} & \textbf{Avg.} \\
    \midrule[\heavyrulewidth]
    \multicolumn{11}{c}{\textbf{Meta-Llama-3-8B}} \\
    Base & $51.8$ & $54.2$ & $12.0$ & $37.2$ & $38.8$ & $55.9$ & $65.2$ & $13.5$ & $42.6$ & $44.3$ \\
    \cdashlinelr{1-11}
    V. SFT & $53.5$ & $61.5$ & $18.3$ & $42.5$ & $44.0$ & $60.7$ & $71.5$ & $21.5$ & $48.9$ & $50.7$ \\
    \quad\textit{+ bal-loader} & $54.9$ & $62.4$ & $19.2$ & $44.6$ & $45.3$ & $60.2$ & $72.3$ & $22.0$ & $50.5$ & $51.3$ \\
    \cdashlinelr{1-11}
    MGDA & $56.4$ & $63.7$ & $21.6$ & $45.5$ & $46.8$ & $63.4$ & $72.7$ & $24.3$ & $50.6$ & $52.8$ \\
    CAGrad & $56.6$ & $63.9$ & $20.2$ & $43.9$ & $46.2$ & $62.4$ & $72.8$ & $23.8$ & $50.5$ & $52.4$ \\
    PCGrad & $56.8$ & $64.3$ & $21.4$ & $45.2$ & $46.9$ & $63.7$ & $73.6$ & $24.5$ & $49.0$ & $52.7$ \\
    \midrule[\heavyrulewidth]
    \multicolumn{11}{c}{\textbf{Llama-3.1-8B}} \\
    Base & $59.1$ & $60.0$ & $13.2$ & $41.8$ & $43.5$ & $67.8$ & $70.0$ & $13.8$ & $48.9$ & $50.1$ \\
    \cdashlinelr{1-11}
    V. SFT & $53.1$ & $61.3$ & $20.3$ & $44.2$ & $44.7$ & $61.3$ & $72.6$ & $22.8$ & $50.5$ & $51.8$ \\
    \quad\textit{+ bal-loader} & $56.6$ & $63.1$ & $20.5$ & $44.1$ & $46.1$ & $62.1$ & $73.5$ & $22.7$ & $52.3$ & $52.6$ \\
    \cdashlinelr{1-11}
    MGDA & $57.1$ & $63.9$ & $22.4$ & $45.2$ & $47.2$ & $65.0$ & $74.1$ & $25.8$ & $51.7$ & $54.2$ \\
    CAGrad & $57.2$ & $63.7$ & $22.9$ & $45.6$ & $47.3$ & $63.9$ & $74.4$ & $26.6$ & $52.1$ & $54.3$ \\
    PCGrad & $57.0$ & $63.5$ & $22.4$ & $44.8$ & $46.9$ & $63.2$ & $74.4$ & $26.1$ & $52.6$ & $54.1$ \\
    \midrule[\heavyrulewidth]
    \multicolumn{11}{c}{\textbf{Qwen3-4B-Base}} \\
    Base & $75.5$ & $81.7$ & $39.4$ & $56.9$ & $63.4$ & $83.8$ & $93.3$ & $44.0$ & $63.6$ & $71.2$ \\
    \cdashlinelr{1-11}
    V. SFT & $69.8$ & $78.9$ & $43.4$ & $55.9$ & $62.0$ & $76.0$ & $88.8$ & $45.8$ & $62.3$ & $68.2$ \\
    \quad\textit{+ bal-loader} & $70.3$ & $77.8$ & $43.5$ & $57.2$ & $62.2$ & $75.5$ & $88.1$ & $45.2$ & $63.7$ & $68.1$ \\
    \cdashlinelr{1-11}
    MGDA & $72.2$ & $80.1$ & $44.4$ & $57.7$ & $63.6$ & $79.3$ & $90.5$ & $47.4$ & $66.0$ & $70.8$ \\
    CAGrad & $72.6$ & $80.7$ & $44.9$ & $58.0$ & $64.0$ & $79.4$ & $90.6$ & $47.2$ & $65.5$ & $70.7$ \\
    PCGrad & $73.0$ & $80.1$ & $44.6$ & $58.4$ & $64.0$ & $79.4$ & $90.6$ & $46.1$ & $64.0$ & $70.0$ \\
    \midrule[\heavyrulewidth]
    \multicolumn{11}{c}{\textbf{Qwen3-8B-Base}} \\
    Base & $81.8$ & $85.7$ & $44.5$ & $61.7$ & $68.4$ & $85.0$ & $95.1$ & $47.6$ & $70.0$ & $74.4$ \\
    \cdashlinelr{1-11}
    V. SFT & $75.3$ & $84.8$ & $46.8$ & $63.4$ & $67.6$ & $80.6$ & $93.4$ & $49.0$ & $69.1$ & $73.0$ \\
    \quad\textit{+ bal-loader} & $76.6$ & $84.6$ & $47.1$ & $64.1$ & $68.1$ & $82.2$ & $93.0$ & $48.6$ & $69.5$ & $73.4$ \\
    \cdashlinelr{1-11}
    MGDA & $77.9$ & $85.3$ & $50.0$ & $64.2$ & $69.3$ & $83.9$ & $94.0$ & $50.5$ & $70.8$ & $74.8$ \\
    CAGrad & $78.4$ & $85.2$ & $49.0$ & $65.2$ & $69.5$ & $83.6$ & $94.0$ & $49.8$ & $71.9$ & $74.8$ \\
    PCGrad & $78.2$ & $85.4$ & $48.4$ & $64.6$ & $69.2$ & $83.8$ & $94.0$ & $49.1$ & $70.3$ & $74.3$ \\
    \bottomrule[\heavyrulewidth]
    \end{tabular}
    }
\end{table*}

\begin{table*}[t!]
    \centering
    \caption{Results on the SEA-HELM FLORES-200 benchmark (higher is better).}
    \label{tab:wmt24pp_flores}
    \resizebox{\textwidth}{!}{
    \renewcommand{\arraystretch}{1.2}
    \setlength{\tabcolsep}{3pt}
    \begin{tabular}{l*{10}{>{\centering\arraybackslash}p{1.4cm}}}
    \toprule
    \multirow{2}{*}{\textbf{Method}} & \multicolumn{5}{c}{\textit{X} $\rightarrow$ \textit{English}} & \multicolumn{5}{c}{\textit{English} $\rightarrow$ \textit{X}} \\
    \cmidrule(r){2-6} \cmidrule(r){7-11}
    & \textbf{VI} & \textbf{ID} & \textbf{TH} & \textbf{TA} & \textbf{Avg.} & \textbf{VI} & \textbf{ID} & \textbf{TH} & \textbf{TA} & \textbf{Avg.} \\
    \midrule[\heavyrulewidth]
    \multicolumn{11}{c}{\textbf{Meta-Llama-3-8B}} \\
    Base & $64.5$ & $42.1$ & $40.5$ & $29.0$ & $44.0$ & $37.5$ & $68.0$ & $52.3$ & $27.4$ & $46.3$ \\
    \cdashlinelr{1-11}
    V. SFT & $85.3$ & $85.9$ & $85.7$ & $72.9$ & $82.4$ & $88.1$ & $88.4$ & $81.5$ & $68.5$ & $81.6$ \\
    \quad\textit{+ bal-loader} & $85.4$ & $84.9$ & $84.8$ & $72.9$ & $82.0$ & $88.2$ & $88.5$ & $81.9$ & $69.0$ & $81.9$ \\
    \cdashlinelr{1-11}
    MGDA & $88.5$ & $89.1$ & $88.4$ & $77.2$ & $85.8$ & $88.3$ & $88.5$ & $82.6$ & $70.7$ & $82.5$ \\
    CAGrad & $88.2$ & $88.7$ & $88.0$ & $76.0$ & $85.2$ & $88.4$ & $88.8$ & $82.7$ & $68.9$ & $82.2$ \\
    PCGrad & $88.3$ & $88.4$ & $87.9$ & $75.8$ & $85.1$ & $89.1$ & $89.4$ & $82.3$ & $68.9$ & $82.4$ \\
    \midrule[\heavyrulewidth]
    \multicolumn{11}{c}{\textbf{Llama-3.1-8B}} \\
    Base & $80.7$ & $37.2$ & $39.1$ & $25.0$ & $45.5$ & $39.8$ & $72.3$ & $52.6$ & $31.1$ & $48.9$ \\
    \cdashlinelr{1-11}
    V. SFT & $85.9$ & $85.7$ & $85.1$ & $74.2$ & $82.7$ & $88.8$ & $89.1$ & $81.1$ & $68.5$ & $81.9$ \\
    \quad\textit{+ bal-loader} & $85.6$ & $85.5$ & $84.8$ & $73.7$ & $82.4$ & $88.4$ & $88.8$ & $80.7$ & $68.6$ & $81.6$ \\
    \cdashlinelr{1-11}
    MGDA & $88.3$ & $89.2$ & $88.2$ & $77.7$ & $85.8$ & $88.8$ & $89.6$ & $81.4$ & $71.0$ & $82.7$ \\
    CAGrad & $88.7$ & $88.9$ & $88.3$ & $76.8$ & $85.6$ & $89.3$ & $89.7$ & $81.6$ & $71.0$ & $82.9$ \\
    PCGrad & $88.3$ & $88.5$ & $87.9$ & $76.9$ & $85.4$ & $89.6$ & $89.8$ & $81.8$ & $69.1$ & $82.6$ \\
    \midrule[\heavyrulewidth]
    \multicolumn{11}{c}{\textbf{Qwen3-4B-Base}} \\
    Base & $89.8$ & $90.3$ & $70.7$ & $75.4$ & $81.5$ & $85.2$ & $83.1$ & $80.8$ & $49.5$ & $74.7$ \\
    \cdashlinelr{1-11}
    V. SFT & $84.6$ & $84.8$ & $83.9$ & $66.4$ & $79.9$ & $86.4$ & $85.7$ & $79.2$ & $62.4$ & $78.4$ \\
    \quad\textit{+ bal-loader} & $84.6$ & $85.1$ & $83.0$ & $67.5$ & $80.1$ & $86.3$ & $85.1$ & $79.5$ & $63.0$ & $78.5$ \\
    \cdashlinelr{1-11}
    MGDA & $88.3$ & $88.5$ & $87.4$ & $72.6$ & $84.2$ & $87.4$ & $86.2$ & $80.7$ & $61.7$ & $79.0$ \\
    CAGrad & $87.8$ & $88.2$ & $87.4$ & $70.7$ & $83.5$ & $87.6$ & $86.4$ & $81.6$ & $62.6$ & $79.5$ \\
    PCGrad & $87.7$ & $88.2$ & $86.8$ & $71.3$ & $83.5$ & $88.2$ & $87.3$ & $81.2$ & $63.4$ & $80.0$ \\
    \midrule[\heavyrulewidth]
    \multicolumn{11}{c}{\textbf{Qwen3-8B-Base}} \\
    Base & $90.5$ & $90.7$ & $85.1$ & $75.2$ & $85.4$ & $88.9$ & $88.3$ & $85.6$ & $68.6$ & $82.9$ \\
    \cdashlinelr{1-11}
    V. SFT & $86.7$ & $86.8$ & $86.7$ & $73.2$ & $83.4$ & $89.2$ & $88.8$ & $83.2$ & $69.8$ & $82.8$ \\
    \quad\textit{+ bal-loader} & $86.4$ & $86.6$ & $86.8$ & $72.4$ & $83.1$ & $89.0$ & $88.9$ & $83.2$ & $70.1$ & $82.8$ \\
    \cdashlinelr{1-11}
    MGDA & $89.4$ & $89.9$ & $89.7$ & $78.0$ & $86.7$ & $90.1$ & $90.0$ & $84.5$ & $72.0$ & $84.1$ \\
    CAGrad & $88.8$ & $89.6$ & $89.6$ & $77.0$ & $86.3$ & $90.1$ & $90.1$ & $84.7$ & $71.3$ & $84.1$ \\
    PCGrad & $88.9$ & $89.7$ & $89.0$ & $77.1$ & $86.2$ & $90.4$ & $90.4$ & $84.7$ & $71.4$ & $84.2$ \\
    \bottomrule[\heavyrulewidth]
    \end{tabular}
    }
\end{table*}

\paragraph{Seen-Language Performance.}
Table \ref{tab:main_compact_aligned} shows that every Bucket-Level MOO variant improves the average seen-language performance over both Vanilla SFT and its balanced-loader control across all four base models. Across all 12 model--method pairs, the gains over Vanilla SFT range from $+1.6$ to $+2.9$ points; against the balanced-loader control, which is the stronger of the two baselines in this setting, MOO consistently adds $+0.8$ to $+1.8$ points. Thus, balanced sampling alone does not account for the improvements. The benefits also extend across capabilities: compared with the balanced-loader baseline, MGDA improves PolyMath by $+2.9$ on \texttt{Qwen3-8B-Base}, while PCGrad improves BELEBELE by $+2.7$ on \texttt{Qwen3-4B-Base}.

\paragraph{Zero-Shot Transfer to Unseen Languages.}
Although the balanced loader changes the Vanilla-SFT averages by only $-0.1$ to $+0.8$ points, all Bucket-Level MOO variants outperform this control, with gains ranging from $+0.9$ to $+2.7$. The largest improvement is obtained by MGDA on \texttt{Qwen3-4B-Base}, increasing the unseen-language average from 68.1 to 70.8. These gains span both knowledge and reasoning: MGDA improves MMLU by $+2.3$ on \texttt{Qwen3-4B-Base}, while CAGrad improves PolyMath by $+3.9$ on \texttt{Llama-3.1-8B}. Conflict-aware optimization therefore promotes transfer beyond the languages observed during fine-tuning.

\paragraph{Mitigating Catastrophic Forgetting.}
The results further reveal that Vanilla SFT can trade reasoning gains for substantial losses in pretrained reading comprehension. For example, the seen-language BELEBELE score of \texttt{Llama-3.1-8B} falls from 59.1 to 53.1 after Vanilla SFT. Balanced loading partially restores it to 56.6, whereas CAGrad further raises it to 57.2. Similarly, on unseen BELEBELE, \texttt{Qwen3-8B-Base} drops from 85.0 to 80.6; balanced loading recovers 82.2, and MGDA reaches 83.9. This preservation is also visible in the aggregate: the seen-language average of \texttt{Qwen3-8B-Base} decreases from 68.4 to 67.6 with Vanilla SFT, recovers to 68.1 with balanced loading, and reaches 69.5 with CAGrad. Recovery is not always complete: on unseen languages, \texttt{Qwen3-4B-Base} reaches $70.8$ with MGDA against a pretrained average of $71.2$. Bucket-Level MOO therefore mitigates destructive overwriting more effectively than balanced sampling alone, while retaining the benefits of multilingual adaptation.

\subsection{Multilingual Machine Translation on WMT24++}
\label{sec:wmt24pp}

\paragraph{Translation Quality.}
Table~\ref{tab:wmt24pp_flores} shows that Bucket-Level MOO consistently improves multilingual translation over Vanilla SFT across all four backbones and both translation directions. The improvement is especially clear for \textit{X}$\rightarrow$\textit{English}, where the best MOO variant gains $+3.1$ to $+4.3$ average points over Vanilla SFT. The corresponding best-variant gains for \textit{English}$\rightarrow$\textit{X} are smaller but still consistent, ranging from $+0.9$ to $+1.6$ points, and every variant improves over Vanilla SFT in both directions. In contrast, the balanced loader performs nearly identically to Vanilla SFT. 

\paragraph{Preserving Pretrained Translation Ability.}
Vanilla SFT can degrade the strong pretrained translation ability of Qwen3 models. For instance, the \textit{X}$\rightarrow$\textit{English} average drops from $81.5$ to $79.9$ on \texttt{Qwen3-4B-Base} and from $85.4$ to $83.4$ on \texttt{Qwen3-8B-Base}. Bucket-Level MOO reverses this trend, with MGDA improving the corresponding scores to $84.2$ and $86.7$. These results indicate that conflict-aware optimization not only improves weaker backbones after SFT, but also better preserves the multilingual capabilities already present in stronger pretrained models.

\section{Ablations: Partition Schemes and Bucket Sizes}
\label{sec:moo_ablation}

\begin{table*}[t!]
    \centering
    \caption{Effects of partition schemes and bucket sizes on LIMA--s1 performance.}
    \label{tab:bucket_performance}
    \resizebox{\textwidth}{!}{
    \renewcommand{\arraystretch}{1.2}
    \setlength{\tabcolsep}{3pt}
    \begin{tabular}{l*{12}{>{\centering\arraybackslash}m{1.2cm}}}
    \toprule
    \multirow{2}{*}{\textbf{Method}} & \multicolumn{6}{c}{\textit{8 Seen Languages}} & \multicolumn{6}{c}{\textit{5 Unseen Languages}} \\
    \cmidrule(r){2-7} \cmidrule(r){8-13}
     & \textbf{Layer-Aligned} & \textbf{100M} & \textbf{250M} & \textbf{500M} & \textbf{1B} & \textbf{Global} & \textbf{Layer-Aligned} & \textbf{100M} & \textbf{250M} & \textbf{500M} & \textbf{1B} & \textbf{Global} \\
    \midrule[\heavyrulewidth]
    \multicolumn{13}{c}{\textbf{Qwen3-4B-Base}} \\
    MGDA & $63.8$ & $63.9$ & $63.3$ & $63.6$ & $64.1$ & $63.3$ & $71.5$ & $71.1$ & $71.0$ & $70.8$ & $70.9$ & $70.2$ \\
    CAGrad & $63.9$ & $63.3$ & $63.7$ & $64.0$ & $63.1$ & $63.1$ & $71.2$ & $70.5$ & $71.2$ & $70.7$ & $70.6$ & $70.0$ \\
    PCGrad & $63.6$ & $63.2$ & $63.8$ & $64.0$ & $63.8$ & $63.6$ & $71.0$ & $70.3$ & $71.2$ & $70.0$ & $70.7$ & $70.8$ \\
    \midrule[\heavyrulewidth]
    \multicolumn{13}{c}{\textbf{Qwen3-8B-Base}} \\
    MGDA & $68.9$ & $69.5$ & $70.0$ & $69.3$ & $69.2$ & $69.2$ & $74.9$ & $75.0$ & $75.3$ & $74.8$ & $75.3$ & $75.0$ \\
    CAGrad & $69.3$ & $69.5$ & $69.2$ & $69.5$ & $69.2$ & $69.4$ & $74.9$ & $75.3$ & $74.9$ & $74.8$ & $74.8$ & $75.2$ \\
    PCGrad & $69.2$ & $69.8$ & $69.6$ & $69.2$ & $69.7$ & $69.3$ & $75.1$ & $75.4$ & $75.3$ & $74.3$ & $74.5$ & $75.1$ \\
    \bottomrule[\heavyrulewidth]
    \end{tabular}
    }
\end{table*}
    
\begin{table*}[t!]
    \centering
    \caption{Peak VRAM (GB) of training under different partition schemes and bucket sizes. All three Bucket-Level MOO variants share an identical footprint.}
    \label{tab:bucket_vram}
    \renewcommand{\arraystretch}{1.2}
    \setlength{\tabcolsep}{6pt}
    \resizebox{317.5pt}{!}{%
    \begin{tabular}{lccccc@{\hspace{1.6em}}c@{\hspace{1.6em}}c}
    \toprule
    \multirow{2}{*}{\textbf{Model}} & \multicolumn{5}{c}{\textbf{Bucket-Level MOO}} & \multirow{2}{*}{\textbf{Global MOO}} & \multirow{2}{*}{\textbf{V. SFT}} \\
    \cmidrule(lr){2-6}
    & \textbf{100M} & \textbf{250M} & \textbf{500M} & \textbf{1B} & \textbf{Layer-Aligned} & & \\
    \midrule
    \texttt{Qwen3-4B-Base} & $34.9$ & $34.5$ & $39.1$ & $39.3$ & $36.1$ & $55.2$ & $36.3$ \\
    \texttt{Qwen3-8B-Base} & $60.7$ & $60.9$ & $61.3$ & $61.7$ & $61.4$ & $85.9$ & $60.5$ \\
    \bottomrule
    \end{tabular}}
\end{table*}

Using models fine-tuned on LIMA--s1, we study how the granularity of conflict resolution affects quality and cost. Global (conventional) MOO is Bucket-Level MOO on the trivial partition $\mathcal{P}=\{\bm{\theta}\}$. In our implementation, Global MOO therefore inherits the memory-efficient single-buffer exchange of Bucket-Level MOO instead of naively gathering all $T$ gradients (Appendix~\ref{app:global_moo}).

\paragraph{Partition Schemes.}
On \texttt{Qwen3-4B-Base} and \texttt{Qwen3-8B-Base}, we compare layer-aligned partitions, system-oriented fixed-size buckets containing 100M--1B gradient elements, and a global partition. Layer-aligned partitions follow architectural boundaries, whereas fixed-size buckets reflect the contiguous gradient buffers used by distributed training systems. We adopt 500M as the default, matching DeepSpeed's default reduction bucket size, and do not tune it further. Smaller buckets reduce temporary memory but introduce more collective operations; larger buckets amortize per-collective overhead at a higher memory cost.

\paragraph{Insensitive Quality to Granularity.}
Table~\ref{tab:bucket_performance} varies the partition over an order of magnitude, from single decoder blocks to the entire model, and the outcome barely moves. Across all methods and both language distributions, performance spans at most $1.3$ points on \texttt{Qwen3-4B-Base} (median $1.1$) and $1.1$ points on \texttt{Qwen3-8B-Base} (median $0.6$). Layer-aligned partitions fall inside the same range as the fixed-size buckets at both scales, so localized resolution needs no alignment with Transformer-layer boundaries and can run directly on the contiguous gradient buffers distributed training already produces. Crucially, this variation is smaller than the effect being measured: all $72$ configurations in the table exceed the balanced-loader baseline of Table~\ref{tab:main_compact_aligned}, by at least $0.8$ points. Which partition one picks therefore never decides whether Bucket-Level MOO helps.

\paragraph{Granularity versus Global MOO.}
The invariance extends to the trivial partition. Averaged over the three solvers and the five partitioned settings, partitioned resolution leads on both splits of \texttt{Qwen3-4B-Base}, by $0.3$ points on the seen languages and $0.5$ on the unseen; on \texttt{Qwen3-8B-Base} the two become indistinguishable, at $0.1$ and $-0.1$. Conflict resolution is therefore granularity-invariant: confining it to buckets matches a full-model decision at both scales, and the match tightens as the backbone grows.

\paragraph{Memory Efficiency.}
By maintaining a single bucket rather than full-model gradients per objective, Bucket-Level MOO drastically curtails memory overhead while keeping peak consumption virtually invariant to bucket size (Table~\ref{tab:bucket_vram}). Across a tenfold range of bucket sizes and layer-aligned partitions, peak VRAM remains stable at $34.5$--$39.3$~GB on \texttt{Qwen3-4B-Base} and $60.7$--$61.7$~GB on \texttt{Qwen3-8B-Base}. Relative to Global MOO, the 500M default cuts peak VRAM by $29.2\%$ on 4B ($39.1$ vs.\ $55.2$~GB) and by $28.6\%$ on 8B ($61.3$ vs.\ $85.9$~GB), with all three solvers sharing the exact same footprint. Crucially, Bucket-Level MOO achieves near-parity with Vanilla SFT ($36.3$~GB on 4B, $60.5$~GB on 8B). The small gap is cached memory for the transient bucket buffers, not live tensors.

\section{Related Work}

\paragraph{Multilingual LLMs and Negative Interference.}
Modern LLMs project diverse languages into a shared parameter space \citep{chen2025the, yang2025qwen3technicalreport}, where these languages occupy distinct representational regions \citep{chang-etal-2022-geometry, xu2026languagethoughtshapesoutput}. Structurally, models often exhibit a division of labor: early and late layers handle multilingual specializations, while middle layers serve as language-universal mechanisms \citep{zhao2024how, bandarkar2026multilingual}. While targeted activation interventions improve reasoning \citep{zhao2026when}, cross-lingual generalization is predominantly driven by data-centric strategies like translating instruction corpora \citep{pan-etal-2024-g, chen-etal-2024-monolingual}, self-distillation \citep{zhang-etal-2024-enhancing-multilingual}, and multi-task tuning \citep{muennighoff-etal-2023-crosslingual}. However, forcing disparate linguistic distributions into a static parameter space inherently triggers \textit{negative interference}, where optimizing for one language disproportionately degrades others \citep{conneau-etal-2020-unsupervised, wang-etal-2020-negative}.

\paragraph{Gradient Conflicts and Scalable Optimization in Distributed LLMs.}
From an optimization perspective, negative interference stems from conflicting gradients across language-specific objectives, which vary across the model: GradVac \citep{wang2021gradient} manipulates them per layer and per weight matrix, and Recon \citep{shi2023recon} makes the most conflicting layers task-specific. Loss-weighting schemes \citep{kendall2018multi, chen2018gradnorm} reweight objectives on the fly; others modify the architecture \citep{zhang-etal-2025-less} or stage the fine-tuning \citep{ye2025synergy}. Gradient-based MOO acts on the gradients themselves: MGDA \citep{desideri2012multiple, NeurIPS2018_Sener_Koltun}, PCGrad \citep{yu2020gradient}, GradDrop \citep{chen2020just}, CAGrad \citep{liu2021conflict}, and Nash-MTL \citep{navon2022bargaining}, with multilingual uses that curb forgetting \citep{mao2022lessforgetting} or filter preference data \citep{li2026congrad}, while RP-MGDA \citep{hu2025leveraging} refines Pareto stationarity over sparsity-derived partitions. At LLM scale, however, FSDP \citep{zhao2023pytorchfsdpexperiencesscaling} and DeepSpeed ZeRO \citep{rajbhandari2020zeromemoryoptimizationstraining} shard gradients, so conventional MOO must hold each full unreduced objective gradient until the update is formed. Bucket-Level MOO instead applies such methods to the gradient buckets the system already creates, with a matching theoretical analysis of the resulting updates.

\section{Conclusion}

We introduced Bucket-Level MOO, a protocol that resolves gradient conflicts within the parameter partitions distributed training already creates; its novelty lies not in a new MOO solver, but in the granularity at which existing ones are applied. Theoretically, we formalize \textit{bucketwise Pareto stationarity}, prove its necessity for Pareto optimality under arbitrary partitions, and establish descent and convergence guarantees for representative solvers. Empirically, the three-variant average exceeds both baselines in all 16 backbone--split cells, gaining $+0.8$ to $+3.8$ points over Vanilla SFT while mitigating forgetting of pretrained capabilities. Conflict resolution also proves granularity-invariant, with quality varying by at most $1.3$ points across partitions and scales. Bucketing thus matches Global (conventional) MOO quality while reducing peak memory by $29\%$ ($16.1$~GB on 4B and $24.6$~GB on 8B) with zero live allocated memory overhead over Vanilla SFT.

Several limitations bound this reading. Our guarantees cover exact deterministic updates for the MGDA and CAGrad variants, leaving PCGrad and the stochastic regime used in practice outside the analysis. We also instantiate the framework with only three solvers; others, such as GradDrop \citep{chen2020just} and Nash-MTL \citep{navon2022bargaining}, should be explored. Furthermore, all experiments fine-tune pre-trained models, so the behavior of Bucket-Level MOO when training multilingually from scratch remains unexamined. Our objectives are languages throughout; extending the framework to other multi-task learning settings, such as multi-domain, multi-modal, or preference-based post-training, remains open. Nevertheless, we remark that these qualifications bound the scope of our analysis rather than the protocol itself. Finally, we hope that treating conflict-resolution granularity as both an optimization and a systems choice, rather than an implementation detail, proves useful wherever objectives are trained jointly at scale.

\subsection*{AI use statement}

We used generative AI tools for language editing, LaTeX formatting, and assistance with code and experimental analysis. We have reviewed all AI-assisted text, claims, and artifacts. We take responsibility for the final content of this work, including text, claims, and artifacts produced with the aid of generative AI.



\subsection*{Reproducibility statement}

Our code is publicly available at \url{https://github.com/iNLP-Lab/BK-MOO}. Training data sources, evaluation benchmarks, and implementation details are described in Section~\ref{sec:BK-MOO} and Section~\ref{sec:exp}, with additional definitions, proofs, and experimental details in the appendix.

\bibliography{iclr2027_conference}
\bibliographystyle{iclr2027_conference}

\appendix

\section{Definitions and Related MOO Results}
\label{app:formal}

\begin{figure}[t!]
    \centering
    \includegraphics[width=0.99\linewidth]{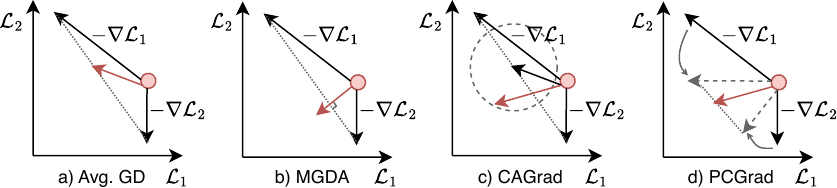}
    
    \caption{Gradient update rules under conflict. Averaging may reduce one loss at the expense of another; MGDA selects the minimum-norm point in the gradients' convex hull; CAGrad optimizes worst-case alignment within a trust region around the average gradient; and PCGrad removes pairwise conflicting components by projection.}
    \label{fig:MGDA}
\end{figure}

\subsection{Conventional Gradient-Based MOO Methods}
\label{app:conventional_moo}

Let \(\mathbf{g}_t=\nabla_{\bm{\theta}}\mathcal{L}_t\) denote the full-model gradient of objective \(t\), let \(\Delta=\{\bm{\alpha}\geq\mathbf{0}:\sum_{t=1}^T\alpha_t=1\}\), and let \(\mathbf{g}_{avg}=T^{-1}\sum_{t=1}^T\mathbf{g}_t\).

\paragraph{Multiple Gradient Descent Algorithm (MGDA).}
MGDA \citep{fliege2000steepest,desideri2012multiple} selects the minimum-norm convex combination of objective gradients:
\begin{align}
    \bm{\alpha}^{*}
    &\in\argmin_{\bm{\alpha}\in\Delta}
    \left\|\sum_{t=1}^{T}\alpha_t\mathbf{g}_t\right\|_2^2,
    &
    \mathbf{d}^{\mathrm{MGDA}}
    &=\sum_{t=1}^{T}\alpha_t^{*}\mathbf{g}_t.
    \label{eq:mgda}
\end{align}
A nonzero \(-\mathbf{d}^{\mathrm{MGDA}}\) is a common descent direction; a zero direction certifies Pareto stationarity.

\paragraph{Conflict-Averse Gradient Descent (CAGrad).}
CAGrad \citep{liu2021conflict} maximizes worst-case alignment within a trust region around the average gradient:
\begin{align}
    \mathbf{d}^{\mathrm{CAGrad}}
    =\argmax_{\mathbf{d}}\min_{t}\langle\mathbf{g}_t,\mathbf{d}\rangle
    \quad\mathrm{s.t.}\quad
    \|\mathbf{d}-\mathbf{g}_{avg}\|_2\leq c\|\mathbf{g}_{avg}\|_2,
    \qquad 0\leq c<1.
    \label{eq:cagrad_global}
\end{align}
The trust region balances conflict aversion with progress on the average objective.

\paragraph{Projected Conflicting Gradients (PCGrad).}
PCGrad \citep{yu2020gradient} projects away pairwise conflicts. When \(\langle\mathbf{g}_i,\mathbf{g}_j\rangle<0\), it replaces
\begin{align}
    \mathbf{g}^{\mathrm{PC}}_i
    =\mathbf{g}_i-
    \frac{\langle\mathbf{g}_i,\mathbf{g}_j\rangle}{\|\mathbf{g}_j\|_2^2}\mathbf{g}_j,
    \qquad
    \mathbf{d}^{\mathrm{PCGrad}}=\frac{1}{T}\sum_{i=1}^{T}\mathbf{g}^{\mathrm{PC}}_i.
    \label{eq:pcgrad_global}
\end{align}
For multiple objectives, PCGrad applies these projections successively, typically in shuffled order, and then aggregates the adjusted gradients.

\subsection{Optimality and Stationarity}

\begin{definition}[Pareto Optimality]
    A point $\bm{\theta}^*$ is \textbf{Pareto optimal} if no $\bm{\theta}$ satisfies $\mathcal{L}_t(\bm{\theta})\leq\mathcal{L}_t(\bm{\theta}^*)$ for every $t\in[T]$ with strict inequality for at least one objective.
\end{definition}

\begin{definition}[Pareto Stationarity]
    \label{def:ps}
    A point $\bm{\theta}^*$ is \textbf{Pareto stationary} (PS) if
    \begin{align}
        \mathbf{0}\in\operatorname{conv}\left\{\nabla_{\bm{\theta}}\mathcal{L}_t(\bm{\theta}^*):t\in[T]\right\}.
    \end{align}
    Equivalently, some $\bm{\alpha}\in\Delta$ satisfies $\sum_{t=1}^T\alpha_t\nabla_{\bm{\theta}}\mathcal{L}_t(\bm{\theta}^*)=\mathbf{0}$.
\end{definition}

\paragraph{Bucketwise stationarity.}
Fix a partition $\mathcal{P}=\{\bm{\theta}_{(1)},\dots,\bm{\theta}_{(K)}\}$ of the trainable parameters into nonempty, disjoint buckets.

\begin{definition}[Bucketwise Pareto Stationarity]
    \label{def:BPS}
    A point $\bm{\theta}^*$ is \textbf{bucketwise Pareto stationary} (BPS) with respect to $\mathcal{P}$ if
    \begin{align}
        \mathbf{0}\in\prod_{k=1}^{K}\operatorname{conv}\left\{\nabla_{\bm{\theta}_{(k)}}\mathcal{L}_t(\bm{\theta}^*):t\in[T]\right\}.
    \end{align}
    Equivalently, each bucket admits weights $\bm{\alpha}^{(k)}\in\Delta$ such that
    \begin{align}
        \sum_{t=1}^{T}\alpha_t^{(k)}\nabla_{\bm{\theta}_{(k)}}\mathcal{L}_t(\bm{\theta}^*)=\mathbf{0}.
    \end{align}
\end{definition}

\begin{theorem}[Necessity of BPS]
    \label{theorem:bps_necessary}
    For an unconstrained parameter domain, if all objectives are continuously differentiable, every Pareto-optimal point is BPS with respect to any fixed partition $\mathcal{P}$.
\end{theorem}
\begin{proof}
Fix $k$ and hold every bucket $\ell\neq k$ at $\bm{\theta}_{(\ell)}^*$. This leaves a $T$-objective problem in $\bm{\theta}_{(k)}$. If $\bm{\theta}_{(k)}^*$ were not Pareto optimal for this restricted problem, replacing bucket $k$ would weakly improve every objective and strictly improve at least one, contradicting Pareto optimality of $\bm{\theta}^*$. The first-order necessary condition for the restricted problem therefore gives $\mathbf{0}\in\operatorname{conv}\{\nabla_{\bm{\theta}_{(k)}}\mathcal{L}_t(\bm{\theta}^*):t\in[T]\}$. Applying this argument to every bucket proves BPS.
\end{proof}

\paragraph{Refined Pareto stationarity.}
Write $\bm{\theta}=(\theta_1,\dots,\theta_d)$ and define the objective--parameter dependency matrix $A\in\{0,1\}^{T\times d}$ by $A_{tj}=1$ if and only if $\mathcal{L}_t$ depends on $\theta_j$. Let
\begin{align}
    \mathcal{P}_{\mathrm{RPS}}
    =\operatorname{RefinedPartition}(A)
    =\{P_1,\dots,P_R\}
\end{align}
be the partition obtained by repeatedly merging parameter coordinates that occur in a common cycle of the bipartite dependency graph until the contracted graph is acyclic \citep{hu2025leveraging}. For each block, let $F(P_r)=\{t\in[T]:A_{tj}=1\text{ for some }j\in P_r\}$, let $\bm{\theta}_{P_r}$ denote the parameter subvector indexed by $P_r$, and let $\Delta_{F(P_r)}$ denote the probability simplex over the objectives in $F(P_r)$.

\begin{definition}[Refined Pareto Stationarity]
    \label{def:RPS}
    A point $\bm{\theta}^*$ is \textbf{refined Pareto stationary} (RPS) if it is blockwise stationary with respect to $\mathcal{P}_{\mathrm{RPS}}$:
    \begin{align}
        \mathbf{0}\in\prod_{r=1}^{R}\operatorname{conv}
        \left\{\nabla_{\bm{\theta}_{P_r}}\mathcal{L}_t(\bm{\theta}^*):t\in F(P_r)\right\}.
    \end{align}
    Equivalently, for each $r$ there exists $\bm{\lambda}^{(r)}\in\Delta_{F(P_r)}$ such that
    \begin{align}
        \sum_{t\in F(P_r)}\lambda_t^{(r)}
        \nabla_{\bm{\theta}_{P_r}}\mathcal{L}_t(\bm{\theta}^*)=\mathbf{0}.
    \end{align}
\end{definition}

\paragraph{Relationships among PO, RPS, PS, and BPS.}

For the dependency-refined partition above, $\mathrm{PO}\subseteq\mathrm{RPS}\subseteq\mathrm{PS}$ \citep{hu2025leveraging}. Moreover, $\mathrm{PS}\subseteq\mathrm{BPS}$ for any fixed bucket partition, because the single global weight vector certifying PS also balances every bucket. Thus,
\begin{align}
    \mathrm{PO}\subseteq\mathrm{RPS}\subseteq\mathrm{PS}\subseteq\mathrm{BPS}.
\end{align}
In a fully shared model, the dependency graph is dense and the refinement collapses to one parameter block, so RPS reduces to PS. BPS generally does not imply PS or RPS because it permits different balancing weights across native buckets. For Bucket-Level MGDA, BPS exactly characterizes zero updates.

\begin{definition}[Conflicting Gradients]
    Gradients $\nabla_{\bm{\theta}}\mathcal{L}_i$ and $\nabla_{\bm{\theta}}\mathcal{L}_j$ \textbf{conflict} if $\langle\nabla_{\bm{\theta}}\mathcal{L}_i,\nabla_{\bm{\theta}}\mathcal{L}_j\rangle<0$, equivalently, if their cosine similarity is negative.
\end{definition}

\begin{theorem}[\cite{desideri2012multiple}]
    Let $\bm{\alpha}$ solve Equation~\ref{eq:mgda}, and set $\mathbf{d}^{\mathrm{MGDA}}=\sum_{t=1}^T\alpha_t\nabla_{\bm{\theta}}\mathcal{L}_t$. Then either
    \begin{enumerate}
        \item[(a)] $\mathbf{d}^{\mathrm{MGDA}}=\mathbf{0}$ and the point is Pareto stationary; or
        \item[(b)] $-\mathbf{d}^{\mathrm{MGDA}}$ is a common descent direction.
    \end{enumerate}
    \label{theorem:mgda_global}
\end{theorem}

\begin{theorem}[\cite{liu2021conflict}]
    
    Let $\mathcal{L}_{avg}=T^{-1}\sum_{t=1}^T\mathcal{L}_t$ be lower bounded, and suppose the objective gradients are $L$-Lipschitz. For exact CAGrad updates with $0\leq c<1$ and fixed step size $0<\eta\leq1/L$,
    \begin{align}
        \sum_{n=0}^{\infty}\|\nabla\mathcal{L}_{avg}(\bm{\theta}^n)\|_2^2<\infty,
    \end{align}
    so $\|\nabla\mathcal{L}_{avg}(\bm{\theta}^n)\|_2\to0$. Consequently, every accumulation point is stationary for $\mathcal{L}_{avg}$ and hence Pareto stationary under uniform weights.
    \label{theorem:cagrad_global}
    
\end{theorem}

\section{Proof of Theorem \ref{theorem:shard-mgda}}
\label{appendix:proof_shard_mgda}

\begin{proof}
For bucket $k$, $\mathbf{d}^{\mathrm{MGDA}}_{(k)}$ is the minimum-norm point in the convex hull of its objective gradients. The projection property of this point gives \citep{desideri2012multiple}
\begin{align}
    \left\langle\nabla_{\bm{\theta}_{(k)}}\mathcal{L}_t,
    \mathbf{d}^{\mathrm{MGDA}}_{(k)}\right\rangle
    \geq \left\|\mathbf{d}^{\mathrm{MGDA}}_{(k)}\right\|_2^2
    \qquad \forall t,k.
    \label{eq:local_prop}
\end{align}
Summing over buckets yields
\begin{align}
    \left\langle\nabla_{\bm{\theta}}\mathcal{L}_t,
    \mathbf{d}^{\mathrm{MGDA}}\right\rangle
    &=\sum_{k=1}^K
    \left\langle\nabla_{\bm{\theta}_{(k)}}\mathcal{L}_t,
    \mathbf{d}^{\mathrm{MGDA}}_{(k)}\right\rangle
    \geq\sum_{k=1}^K\left\|\mathbf{d}^{\mathrm{MGDA}}_{(k)}\right\|_2^2
    =\left\|\mathbf{d}^{\mathrm{MGDA}}\right\|_2^2.
    \label{eq:global_bound}
\end{align}
Because the squared global norm is the sum of the squared bucket norms, $\mathbf{d}^{\mathrm{MGDA}}=\mathbf{0}$ exactly when every bucket direction is zero. By the local MGDA definition, this is equivalent to BPS. Otherwise, Equation~\ref{eq:global_bound} is strictly positive for every objective, so $-\mathbf{d}^{\mathrm{MGDA}}$ is a common descent direction.

For convergence, let $L=\max_tL_t$ and $\gamma=\eta(1-L\eta/2)>0$. Smoothness and Equation~\ref{eq:global_bound} imply, for every $t$,
\begin{align}
\mathcal{L}_t(\bm{\theta}^{n+1})
&\leq \mathcal{L}_t(\bm{\theta}^{n})-\eta\langle\nabla\mathcal{L}_t(\bm{\theta}^{n}),\mathbf{d}^{n}\rangle+\frac{L_t\eta^2}{2}\|\mathbf{d}^{n}\|_2^2\\
&\leq \mathcal{L}_t(\bm{\theta}^{n})-\gamma\|\mathbf{d}^{n}\|_2^2.
\end{align}
Telescoping and lower boundedness imply
\begin{align}
\gamma\sum_{n=0}^{\infty}\|\mathbf{d}^{n}\|_2^2
\leq \mathcal{L}_t(\bm{\theta}^{0})-\inf_{\bm{\theta}}\mathcal{L}_t(\bm{\theta})<\infty,
\end{align}
so $\|\mathbf{d}^{n}\|_2\to0$. For any convergent subsequence $\bm{\theta}^{n_j}\to\bar{\bm{\theta}}$, compactness of the simplices yields a further subsequence on which every $\bm{\alpha}^{(k),n_j}$ converges. Taking limits in $\mathbf{d}^{n_j}_{(k)}=\sum_t\alpha_t^{(k),n_j}\nabla_{\bm{\theta}_{(k)}}\mathcal{L}_t(\bm{\theta}^{n_j})\to\mathbf{0}$ shows that $\bar{\bm{\theta}}$ is BPS.
\end{proof}

\section{Proof of Theorem \ref{theorem:shard-cagrad}}
\label{appendix:proof_shard_cagrad}

\begin{proof}
Let $F=\mathcal{L}_{avg}$, $\mathbf{g}_{(k)}=\nabla_{\bm{\theta}_{(k)}}F$, and $\mathbf{d}_{(k)}=\mathbf{d}^{\mathrm{CAGrad}}_{(k)}$. The local trust-region constraint is
\begin{align}
    \|\mathbf{d}_{(k)}-\mathbf{g}_{(k)}\|_2^2
    \leq c^2\|\mathbf{g}_{(k)}\|_2^2.
    \label{eq:local_cagrad_opt}
\end{align}
Expanding it gives
\begin{align}
    2\langle\mathbf{g}_{(k)},\mathbf{d}_{(k)}\rangle
    \geq\|\mathbf{d}_{(k)}\|_2^2+(1-c^2)\|\mathbf{g}_{(k)}\|_2^2.
    \label{eq:local_cagrad_bound}
\end{align}
After summing over buckets,
\begin{align}
    2\langle\nabla F,\mathbf{d}\rangle
    \geq\|\mathbf{d}\|_2^2+(1-c^2)\|\nabla F\|_2^2.
    \label{eq:cagrad_global_bound}
\end{align}
If $\nabla F\neq\mathbf{0}$, the right-hand side is positive because $c<1$; hence $-\mathbf{d}$ is a descent direction for $F$.

The global update is zero exactly when every $\mathbf{d}_{(k)}$ is zero. In that case, the trust-region constraint gives $\|\mathbf{g}_{(k)}\|_2\leq c\|\mathbf{g}_{(k)}\|_2$, so $\mathbf{g}_{(k)}=\mathbf{0}$ because $c<1$. Conversely, $\mathbf{g}_{(k)}=\mathbf{0}$ forces $\mathbf{d}_{(k)}=\mathbf{0}$. Hence $\mathbf{d}=\mathbf{0}$ if and only if $\nabla F=\mathbf{0}$. The latter condition is PS under uniform global weights and BPS under uniform bucket weights.

For convergence, write $\mathbf{g}^n=\nabla F(\bm{\theta}^n)$. For $0<\eta\leq1/L_F$, smoothness and Equation~\ref{eq:cagrad_global_bound} give
\begin{align}
F(\bm{\theta}^{n+1})
&\leq F(\bm{\theta}^{n})-\eta\langle\mathbf{g}^n,\mathbf{d}^n\rangle+\frac{L_F\eta^2}{2}\|\mathbf{d}^n\|_2^2\\
&\leq F(\bm{\theta}^{n})-\frac{\eta(1-c^2)}{2}\|\mathbf{g}^n\|_2^2.
\end{align}
Lower boundedness of $F$ then implies $\sum_n\|\mathbf{g}^n\|_2^2<\infty$ and $\|\mathbf{g}^n\|_2\to0$. Finally, $\|\mathbf{d}^n-\mathbf{g}^n\|_2\leq c\|\mathbf{g}^n\|_2$ gives $\|\mathbf{d}^n\|_2\leq(1+c)\|\mathbf{g}^n\|_2\to0$.
\end{proof}

\section{Distributed Execution}
\label{app:alg}
\begin{algorithm}[t]
\caption{Bucket-Level MOO: one optimizer step under ZeRO-2 data parallelism. Worker $t$ owns objective $\mathcal{L}_t$.}
\label{alg:bucket_moo}
\begin{algorithmic}[1]
\Require Bucket partition $\mathcal{P}=\{\bm{\theta}_{(k)}\}_{k=1}^{K}$ supplied by the distributed runtime
\Require Per-worker datasets $\{\mathcal{D}_t\}_{t=1}^{T}$; rule $\mathcal{M}\in\{\textsc{MGDA},\textsc{CAGrad},\textsc{PCGrad}\}$
\Require Optimizer \textsc{Opt}; base seed $s$ shared by all workers
\Statex
\State \textbf{on worker} $t$\textbf{:} draw a micro-batch from $\mathcal{D}_t$ and evaluate $\mathcal{L}_t(\bm{\theta})$ \Comment{forward pass}
\For{each bucket $k$ \textbf{as its gradients become ready during backpropagation}}
    \State $\mathbf{g}^{(k)}_t\gets\nabla_{\bm{\theta}_{(k)}}\mathcal{L}_t$
           \Comment{local to worker $t$}
    \State allocate one buffer $\mathbf{b}$ of size $|\bm{\theta}_{(k)}|$
           \Comment{the only memory conflict resolution adds}
    \For{$j=1,\dots,T$}
        \If{$j=t$}
            \State $\mathbf{b}\gets\mathbf{g}^{(k)}_t$
                   \Comment{worker $j$ supplies the payload for its own turn}
        \EndIf
        \State $\mathbf{b}\gets\textsc{Broadcast}\big(\mathbf{b};\ \text{root}=j\big)$
               \Comment{one gradient in flight at a time}
        \State $\rho_j\gets\langle\mathbf{g}^{(k)}_t,\mathbf{b}\rangle$
               \Comment{chunked, \texttt{fp32} per chunk into an \texttt{fp64} accumulator; $\mathbf{b}$ is consumed on arrival}
    \EndFor
    \State free $\mathbf{b}$
           \Comment{peak MOO memory stays $O\big(\max_k|\bm{\theta}_{(k)}|\big)$, independent of $T$}
    \State $R\gets\textsc{All-Gather}(\bm{\rho})$;\quad
           $M\gets\tfrac{1}{2}\big(R+R^{\top}\big)$
           \Comment{$T^2$ \emph{scalars}, identical on every worker; symmetrizing makes $M$ exactly symmetric}
    \State $\bm{\beta}\gets\textsc{Coefficients}_{\mathcal{M}}(M;\,s_k)$,\quad $s_k=(s,\,\text{step},\,k)$
           \Comment{identical on every worker (Alg.~\ref{alg:coeff}); $O(T^3)$ scalar work}
    \State $\mathbf{g}^{(k)}_t\gets\beta_t\,\mathbf{g}^{(k)}_t$;\quad
           $\mathbf{d}_{(k)}\gets\textsc{All-Reduce}\big(\mathbf{g}^{(k)}_t,\ \textsc{Sum}\big)$
           \Comment{scaled in place; worker $t$ contributes only its own term}
    \State keep only worker $t$'s ZeRO-2 partition of $\mathbf{d}_{(k)}$
           \Comment{as for ordinary ZeRO-2 gradients}
\EndFor
\State $\bm{\theta}\gets\textsc{Opt}\big(\bm{\theta},\;\mathbf{d}=[\mathbf{d}_{(1)},\dots,\mathbf{d}_{(K)}]\big)$
       \Comment{per-partition AdamW step, as in ZeRO-2}
\end{algorithmic}
\end{algorithm}

\begin{algorithm}[t]
\caption{$\textsc{Coefficients}_{\mathcal{M}}$. Each rule returns $\bm{\beta}\in\mathbb{R}^{T}$ such that the bucket
update is $\mathbf{d}_{(k)}=\sum_{t=1}^{T}\beta_t\,\mathbf{g}^{(k)}_t$, and depends on the gradients \emph{only}
through the Gram matrix $M$. This is what lets Algorithm~\ref{alg:bucket_moo} assemble $M$ by streaming exchange
instead of materializing the objective gradients.}
\label{alg:coeff}
\begin{algorithmic}[1]
\Require Gram matrix $M\in\mathbb{R}^{T\times T}$ in \texttt{fp64}; simplex $\Delta=\{\bm{\alpha}\geq\mathbf{0}:\sum_t\alpha_t=1\}$; uniform weights $\bm{u}=T^{-1}\mathbf{1}$
\Statex
\Function{Coefficients$_{\textsc{MGDA}}$}{$M$}
    \State $V\Lambda V^{\top}\gets\operatorname{eig}(M)$;\quad $P\gets V\sqrt{\max(\Lambda,0)}$
           \Comment{$PP^{\top}\!=\!M$, so the rows of $P$ are $T$-dimensional proxies for the gradients}
    \State \Return $\argmin_{\bm{\alpha}\in\Delta}\big\|P^{\top}\bm{\alpha}\big\|_2^2$
           \Comment{min-norm point (Eq.~\ref{eq:shard-mgda}), via Frank--Wolfe \citep{frank1956algorithm}}
\EndFunction
\Statex
\Function{Coefficients$_{\textsc{CAGrad}}$}{$M$, $c$}
    \State $\phi\gets c\sqrt{\bm{u}^{\top}M\bm{u}}$
           \Comment{$\phi=c\|\mathbf{g}^{avg}_{(k)}\|_2$}
    \State $\bm{w}^{*}\gets\argmin_{\bm{w}\in\Delta}\;\bm{w}^{\top}M\bm{u}+\phi\sqrt{\bm{w}^{\top}M\bm{w}}$
           \Comment{dual of Equation~\ref{eq:bucket-cagrad} \citep{liu2021conflict}, solved by SLSQP}
    \State \Return $\dfrac{1}{1+c^{2}}\left(\bm{u}+\dfrac{\phi}{\sqrt{\bm{w}^{*\top}M\bm{w}^{*}}}\,\bm{w}^{*}\right)$
           \Comment{including the $1/(1+c^{2})$ rescaling of the reference implementation}
\EndFunction
\Statex
\Function{Coefficients$_{\textsc{PCGrad}}$}{$M$, $s_k$}
    \State $C\gets I_T$;\quad $\pi\gets\textsc{Shuffle}\big([T];\,s_k\big)$
           \Comment{one order, drawn from the shared seed and reused for every $t$}
    \For{$t=1,\dots,T$}
        \For{$j\in\pi$ \textbf{with} $j\neq t$}
            \State $\rho\gets\big(C_{t,:}M\big)_j$
                   \Comment{$\rho=\langle\mathbf{g}^{\mathrm{PC}}_t,\mathbf{g}^{(k)}_j\rangle$}
            \If{$\rho<0$}
                \State $C_{tj}\gets C_{tj}-\rho/M_{jj}$
                       \Comment{row $t$ expands $\mathbf{g}^{\mathrm{PC}}_t$ in the basis $\{\mathbf{g}^{(k)}_j\}_{j=1}^{T}$}
            \EndIf
        \EndFor
    \EndFor
    \State \Return $T^{-1}C^{\top}\mathbf{1}$
           \Comment{averaged column sums, since $\mathbf{d}_{(k)}=T^{-1}\sum_t\mathbf{g}^{\mathrm{PC}}_t$}
\EndFunction
\end{algorithmic}
\end{algorithm}

Algorithm~\ref{alg:bucket_moo} states Bucket-Level MOO as a protocol rather than a solver. As each bucket's gradients become ready during backpropagation, each worker broadcasts that bucket in turn, every other worker reduces the arrival to a scalar inner product through a single reusable buffer, the bucket's conflicts are resolved from the resulting Gram matrix, the resolved update is handed to the standard ZeRO-2 reduction, which leaves each worker only its own partition, and the buffer is released before the next bucket arrives. No stage of the step ever holds a full-model gradient for any objective, and no worker ever holds more than one bucket of gradient beyond its own.

\paragraph{One interface for every rule.}
MGDA, CAGrad, and PCGrad all return an update that is a linear combination $\sum_t \beta_t\mathbf{g}^{(k)}_t$ whose coefficients depend on the bucket gradients only through their pairwise inner products. Algorithm~\ref{alg:coeff} makes this explicit: each rule consumes the $T\times T$ Gram matrix and returns $\bm{\beta}\in\mathbb{R}^{T}$. Three consequences follow. First, the per-bucket solver cost is $O(T^3)$ scalar work plus the $O(T|\bm{\theta}_{(k)}|)$ inner products accumulated during the exchange, independent of how the rule combines gradients; for PCGrad this replaces the $O(T^2)$ vector projections of the naive implementation and explains the small overhead measured in Appendix~\ref{app:time}. Second, because $M$ is the only quantity the solver reads, the gradients themselves never need to coexist: they can be consumed one at a time as they arrive, which is what bounds the memory in Algorithm~\ref{alg:bucket_moo}. Third, substituting a different MOO or multi-task rule requires only a new \textsc{Coefficients} routine, leaving the distributed protocol unchanged, provided the rule is expressible through the Gram matrix; element-wise rules such as GradDrop \citep{chen2020just} fall outside this interface.

\paragraph{Cost.}
Let $|\bm{\theta}|=\sum_k|\bm{\theta}_{(k)}|$. Peak memory attributable to conflict resolution is $O(\max_k|\bm{\theta}_{(k)}|)$, a single buffer, rather than the $O(T|\bm{\theta}|)$ of a naive global implementation that materializes every objective gradient before solving; in particular it does not grow with the number of objectives. The Global MOO baseline of Section~\ref{sec:moo_ablation} is not naive in this sense: it runs the same exchange with a single model-sized bucket, so its extra memory is $O(|\bm{\theta}|)$ (see below). Communication volume, by contrast, is \emph{not} reduced: each worker receives $(T-1)|\bm{\theta}_{(k)}|$ elements per bucket, the same total an All-Gather would move and the same total a global implementation would move. The closing All-Reduce is the collective that DeepSpeed's ZeRO-2 already uses for ordinary gradients before each worker keeps its own partition, so Vanilla SFT performs it as well and the added traffic is the $T$ broadcasts. Bucketing itself does not change the volume, since the same exchange could be applied to full gradients; what bucketing contributes is that the transfers are issued incrementally during the backward pass, bucket by bucket, and keep the transient buffer independent of both $T$ and model size. The transfers do not overlap with computation in our runs: we keep DeepSpeed's ZeRO-2 default of no communication--computation overlap, and each bucket's solve waits for its Gram matrix on the CPU. Halving the volume is possible in principle, by forwarding each gradient only $\lfloor T/2\rfloor$ hops around a ring so that every pair of objectives meets exactly once rather than twice, but we do not implement it here. Table~\ref{tab:runtime} reports the resulting wall-clock overhead.

\paragraph{Global MOO baseline.}\label{app:global_moo}
The Global columns of Tables~\ref{tab:bucket_performance} and~\ref{tab:bucket_vram} use the trivial partition $\mathcal{P}=\{\bm{\theta}\}$ and run through exactly the same code path as every bucketed configuration: the same ZeRO-2 setup, data pipeline, hyperparameters and \textsc{Coefficients} routines. The only change is the ZeRO-2 reduction bucket, which we set to the number of trainable parameters plus one, so that the whole gradient forms a single bucket ($K=1$); the parameter all-gather bucket keeps its default size. Algorithm~\ref{alg:bucket_moo} then streams the $T$ full-model language gradients through one model-sized buffer, solves a single MOO problem for the whole model from the global Gram matrix, and all-reduces the weighted gradients, each worker again keeping only its own partition. Global MOO therefore never holds another worker's gradient beyond that one buffer, but its flattened bucket, its broadcast buffer and DeepSpeed's reduction buffer are all model-sized rather than bounded by the bucket size. This accounts for its footprint in Table~\ref{tab:bucket_vram}: on \texttt{Qwen3-8B-Base}, its peak of $85.9$~GB, whereas every bucketed configuration stays below $62$~GB.

\paragraph{Memory measurement.}
Peak VRAM in Table~\ref{tab:bucket_vram} is the largest device memory reported by \texttt{nvidia-smi} on any worker during training, in GB of $2^{30}$ bytes. It includes memory that PyTorch's caching allocator reserves but does not use. DeepSpeed also sizes its own reduction buffers by the bucket size, so configurations with 100M or 250M buckets can peak slightly below Vanilla SFT, which keeps the 500M default. Live allocations are tighter: on \texttt{Qwen3-4B-Base}, the peak allocated memory (\texttt{torch.cuda.max\_memory\_allocated}) is $25.97$~GB for both Vanilla SFT and Bucket-Level MOO at 500M, and $35.84$~GB for Global MOO, so the difference between the first two in Table~\ref{tab:bucket_vram} is reserved but unused memory.

\paragraph{Numerical guards.}
Algorithm~\ref{alg:coeff} omits the small dampers the implementation carries: the CAGrad and PCGrad divisions add $10^{-8}$ to $\sqrt{\bm{w}^{\top}M\bm{w}}$ and to $M_{jj}$ respectively, so a bucket in which some objective's gradient vanishes cannot produce a division by zero. CAGrad additionally short-circuits to $\bm{u}$ when $c=0$, which is the exact limit of its subproblem.

\paragraph{Practical notes.}
The presentation above assigns one objective per worker, as in all our experiments (eight workers for LIMA--s1 and four for WMT24++), and each optimizer step uses a single micro-batch per worker, without gradient accumulation; with accumulation, $\mathbf{g}^{(k)}_t$ would denote the accumulated bucket gradient for objective $t$ and conflict resolution would be applied once per optimizer step, not once per micro-batch. DeepSpeed's ZeRO-2 reduction divides each bucket by the number of workers before our routine runs; because all three rules are invariant to a common rescaling of the gradients, $\bm{\beta}$ is unchanged and the applied update is $\mathbf{d}/T$, a constant factor that AdamW absorbs. Because every worker sees the same Gram matrix $M$ and solves the same subproblem, the solvers must be deterministic across ranks; \textsc{PCGrad} therefore draws its projection order from a seed shared by all workers. Finally, Theorems~\ref{theorem:shard-mgda} and \ref{theorem:shard-cagrad} analyze the raw direction $\mathbf{d}$ under the exact update $\bm{\theta}\gets\bm{\theta}-\eta\mathbf{d}$; our experiments pass $\mathbf{d}$ to AdamW, whose preconditioning is outside that analysis.

\section{Experiment Details}
\label{appendix:exp}

We fine-tune all models using the \texttt{TRL}\footnote{\url{https://github.com/huggingface/trl}} library with a maximum sequence length of 2,048 tokens and a batch size of 128. All runs use \texttt{bfloat16} precision, the AdamW optimizer, and a cosine learning-rate schedule with a warmup ratio of $0.03$. For the Qwen3 models, we use a peak learning rate of $2\times10^{-5}$ and a weight decay of $0.1$; for the Llama 3 models, we use a peak learning rate of $4\times10^{-6}$ and a weight decay of $0.04$. We train for three epochs and evaluate the final checkpoint. The LIMA--s1 experiments use eight NVIDIA H200 GPUs, whereas the WMT24++ experiments use four. The \textit{+ bal-loader} baseline combines the language gradients with uniform weights and uses a balanced data loader that guarantees that every step loads the language-specific batches in balance. For Bucket-Level CAGrad, we set the conflict-aversion coefficient to $c=0.5$ for all runs, and rescale the resulting direction by $1/(1+c^2)$, following the reference CAGrad implementation; this constant factor is absorbed by AdamW's preconditioning and is omitted from Equation~\ref{eq:bucket-cagrad} and Theorem~\ref{theorem:shard-cagrad} for clarity.

\begin{table}[t!]
    \centering
    \caption{Training time for \texttt{Meta-Llama-3-8B} on LIMA--s1 using eight H200 GPUs.}
    \label{tab:runtime}
    \begingroup
    \tabletextsize
    \setlength{\tabcolsep}{4pt}
    \begin{tabular}{l c c c c}
    \toprule
    Method & V. SFT & MGDA & CAGrad & PCGrad \\ \midrule
    Hours & $0.6$ & $0.8$ & $0.7$ & $0.8$ \\
    \bottomrule
    \end{tabular}
    \endgroup
    \end{table}
    
\section{Training Time}
\label{app:time}

Table \ref{tab:runtime} reports wall-clock training time for \texttt{Meta-Llama-3-8B} on the LIMA--s1 experiment. Vanilla SFT completes in $0.6$ hours, while CAGrad requires $0.7$ hours and MGDA and PCGrad each require $0.8$ hours. Thus, bucket-level conflict resolution adds only $0.1$--$0.2$ hours to this configuration, indicating a modest computational overhead relative to standard SFT.

\begin{table*}[t!]
    \centering
    \caption{Per-language results on the eight seen languages after LIMA--s1 fine-tuning.}
    \label{tab:per_language_seen}
    \resizebox{\textwidth}{!}{
    \renewcommand{\arraystretch}{1.2}
    \setlength{\tabcolsep}{3pt}
    \begin{tabular}{l*{9}{>{\centering\arraybackslash}p{1.4cm}}}
    \toprule
    \textbf{Method} & \textbf{EN} & \textbf{ZH} & \textbf{IT} & \textbf{AR} & \textbf{KO} & \textbf{ID} & \textbf{BN} & \textbf{SW} & \textbf{Avg.} \\
    \midrule[\heavyrulewidth]
    \multicolumn{10}{c}{\textbf{Meta-Llama-3-8B}} \\
    Base & $53.7$ & $39.6$ & $37.1$ & $38.4$ & $40.3$ & $44.5$ & $30.5$ & $26.0$ & $38.8$ \\
    \cdashlinelr{1-10}
    V. SFT & $59.3$ & $45.6$ & $52.5$ & $41.5$ & $42.5$ & $47.7$ & $32.9$ & $29.4$ & $44.0$ \\
    \quad\textit{+ bal-loader} & $61.4$ & $50.3$ & $52.0$ & $43.3$ & $44.1$ & $45.4$ & $34.1$ & $31.4$ & $45.3$ \\
    \cdashlinelr{1-10}
    MGDA & $63.6$ & $50.3$ & $55.2$ & $44.2$ & $44.9$ & $50.9$ & $34.6$ & $30.9$ & $46.8$ \\
    CAGrad & $60.9$ & $48.5$ & $54.3$ & $44.4$ & $45.0$ & $51.2$ & $34.3$ & $30.6$ & $46.2$ \\
    PCGrad & $62.5$ & $50.9$ & $54.2$ & $44.6$ & $45.6$ & $50.8$ & $34.5$ & $32.2$ & $46.9$ \\
    \midrule[\heavyrulewidth]
    \multicolumn{10}{c}{\textbf{Llama-3.1-8B}} \\
    Base & $59.4$ & $48.6$ & $51.4$ & $41.9$ & $44.1$ & $46.9$ & $30.1$ & $25.5$ & $43.5$ \\
    \cdashlinelr{1-10}
    V. SFT & $60.9$ & $48.9$ & $51.9$ & $42.2$ & $43.4$ & $48.2$ & $34.1$ & $28.1$ & $44.7$ \\
    \quad\textit{+ bal-loader} & $61.9$ & $50.2$ & $54.7$ & $42.7$ & $45.2$ & $46.5$ & $36.5$ & $30.9$ & $46.1$ \\
    \cdashlinelr{1-10}
    MGDA & $64.9$ & $52.7$ & $56.6$ & $44.9$ & $44.1$ & $50.4$ & $35.4$ & $28.3$ & $47.2$ \\
    CAGrad & $63.3$ & $54.1$ & $55.7$ & $44.1$ & $45.1$ & $52.0$ & $35.9$ & $28.4$ & $47.3$ \\
    PCGrad & $62.9$ & $51.4$ & $55.8$ & $44.1$ & $45.7$ & $51.1$ & $34.6$ & $29.8$ & $46.9$ \\
    \midrule[\heavyrulewidth]
    \multicolumn{10}{c}{\textbf{Qwen3-4B-Base}} \\
    Base & $75.3$ & $72.3$ & $72.1$ & $65.0$ & $65.4$ & $71.6$ & $55.7$ & $29.5$ & $63.4$ \\
    \cdashlinelr{1-10}
    V. SFT & $75.2$ & $70.4$ & $69.9$ & $62.9$ & $63.2$ & $66.9$ & $52.5$ & $35.1$ & $62.0$ \\
    \quad\textit{+ bal-loader} & $75.8$ & $69.7$ & $70.3$ & $62.9$ & $64.3$ & $65.5$ & $54.7$ & $34.1$ & $62.2$ \\
    \cdashlinelr{1-10}
    MGDA & $77.5$ & $71.2$ & $72.3$ & $63.7$ & $66.6$ & $70.6$ & $55.3$ & $31.3$ & $63.6$ \\
    CAGrad & $77.4$ & $71.8$ & $73.8$ & $62.8$ & $65.0$ & $70.7$ & $56.5$ & $34.4$ & $64.0$ \\
    PCGrad & $78.3$ & $71.2$ & $72.8$ & $64.0$ & $65.7$ & $69.2$ & $55.3$ & $35.8$ & $64.0$ \\
    \midrule[\heavyrulewidth]
    \multicolumn{10}{c}{\textbf{Qwen3-8B-Base}} \\
    Base & $78.4$ & $73.9$ & $75.0$ & $68.7$ & $70.7$ & $75.7$ & $64.6$ & $40.2$ & $68.4$ \\
    \cdashlinelr{1-10}
    V. SFT & $79.6$ & $74.3$ & $73.7$ & $68.5$ & $69.4$ & $73.0$ & $60.1$ & $41.9$ & $67.6$ \\
    \quad\textit{+ bal-loader} & $80.6$ & $73.9$ & $74.9$ & $69.4$ & $69.5$ & $71.2$ & $61.8$ & $43.3$ & $68.1$ \\
    \cdashlinelr{1-10}
    MGDA & $80.9$ & $75.2$ & $76.6$ & $70.9$ & $72.6$ & $73.8$ & $63.0$ & $41.8$ & $69.3$ \\
    CAGrad & $81.1$ & $74.9$ & $76.5$ & $72.2$ & $71.0$ & $73.9$ & $62.0$ & $44.2$ & $69.5$ \\
    PCGrad & $80.5$ & $75.9$ & $76.1$ & $71.1$ & $70.0$ & $73.2$ & $62.4$ & $44.1$ & $69.2$ \\
    \bottomrule[\heavyrulewidth]
    \end{tabular}
    }
    \end{table*}
    
    \begin{table*}[t!]
        \centering
        \caption{Zero-shot results on the five unseen languages after LIMA--s1 fine-tuning.}
        \label{tab:per_language_unseen}
        \resizebox{0.65\textwidth}{!}{
        \renewcommand{\arraystretch}{1.2}
        \setlength{\tabcolsep}{3pt}
        \begin{tabular}{l*{6}{>{\centering\arraybackslash}p{1.4cm}}}
        \toprule
        \textbf{Method} & \textbf{DE} & \textbf{ES} & \textbf{FR} & \textbf{JA} & \textbf{PT} & \textbf{Avg.} \\
        \midrule[\heavyrulewidth]
        \multicolumn{7}{c}{\textbf{Meta-Llama-3-8B}} \\
        Base & $45.2$ & $46.2$ & $46.1$ & $34.3$ & $49.8$ & $44.3$ \\
        \cdashlinelr{1-7}
        V. SFT & $51.1$ & $54.3$ & $51.2$ & $43.8$ & $53.0$ & $50.7$ \\
        \quad\textit{+ bal-loader} & $51.9$ & $53.0$ & $53.2$ & $43.9$ & $54.4$ & $51.3$ \\
        \cdashlinelr{1-7}
        MGDA & $54.4$ & $56.2$ & $54.1$ & $43.7$ & $55.3$ & $52.8$ \\
        CAGrad & $54.0$ & $55.2$ & $54.5$ & $42.8$ & $55.4$ & $52.4$ \\
        PCGrad & $53.9$ & $56.8$ & $54.5$ & $43.7$ & $54.6$ & $52.7$ \\
        \midrule[\heavyrulewidth]
        \multicolumn{7}{c}{\textbf{Llama-3.1-8B}} \\
        Base & $49.2$ & $50.9$ & $53.2$ & $43.5$ & $53.7$ & $50.1$ \\
        \cdashlinelr{1-7}
        V. SFT & $53.7$ & $53.4$ & $53.1$ & $44.4$ & $54.3$ & $51.8$ \\
        \quad\textit{+ bal-loader} & $54.1$ & $54.8$ & $55.0$ & $45.1$ & $54.2$ & $52.6$ \\
        \cdashlinelr{1-7}
        MGDA & $54.8$ & $57.3$ & $55.1$ & $45.9$ & $57.7$ & $54.2$ \\
        CAGrad & $55.0$ & $57.1$ & $57.0$ & $45.4$ & $56.8$ & $54.3$ \\
        PCGrad & $55.2$ & $56.4$ & $56.9$ & $44.3$ & $57.6$ & $54.1$ \\
        \midrule[\heavyrulewidth]
        \multicolumn{7}{c}{\textbf{Qwen3-4B-Base}} \\
        Base & $70.4$ & $74.0$ & $71.7$ & $67.2$ & $72.6$ & $71.2$ \\
        \cdashlinelr{1-7}
        V. SFT & $68.9$ & $71.1$ & $69.3$ & $62.7$ & $69.1$ & $68.2$ \\
        \quad\textit{+ bal-loader} & $68.2$ & $69.8$ & $68.7$ & $63.2$ & $70.7$ & $68.1$ \\
        \cdashlinelr{1-7}
        MGDA & $70.3$ & $73.2$ & $71.1$ & $67.1$ & $72.2$ & $70.8$ \\
        CAGrad & $71.0$ & $72.5$ & $70.8$ & $66.5$ & $72.7$ & $70.7$ \\
        PCGrad & $70.8$ & $72.5$ & $71.0$ & $65.0$ & $70.7$ & $70.0$ \\
        \midrule[\heavyrulewidth]
        \multicolumn{7}{c}{\textbf{Qwen3-8B-Base}} \\
        Base & $75.2$ & $76.5$ & $74.0$ & $70.9$ & $75.4$ & $74.4$ \\
        \cdashlinelr{1-7}
        V. SFT & $73.9$ & $75.2$ & $73.8$ & $68.0$ & $74.1$ & $73.0$ \\
        \quad\textit{+ bal-loader} & $73.0$ & $74.8$ & $75.1$ & $69.8$ & $74.1$ & $73.4$ \\
        \cdashlinelr{1-7}
        MGDA & $75.7$ & $75.8$ & $76.4$ & $70.8$ & $75.3$ & $74.8$ \\
        CAGrad & $76.2$ & $76.7$ & $75.4$ & $70.2$ & $75.7$ & $74.8$ \\
        PCGrad & $73.7$ & $76.3$ & $74.9$ & $70.8$ & $75.8$ & $74.3$ \\
        \bottomrule[\heavyrulewidth]
        \end{tabular}
        }
        \end{table*}

\section{Per-Language Analysis on LIMA--s1}
\label{app:per_language_analysis}

\paragraph{Seen-Language Improvements.}
Table \ref{tab:per_language_seen} shows that, for every backbone, all three Bucket-Level MOO variants outperform both SFT baselines on average. Relative to balanced SFT, MOO improves 79 of the 96 model--language--method comparisons, although the magnitude varies across languages and backbones. Indonesian (\texttt{id}) benefits consistently: CAGrad improves its score by $+5.8$, $+5.5$, $+5.2$, and $+2.7$ points on \texttt{Meta-Llama-3-8B}, \texttt{Llama-3.1-8B}, \texttt{Qwen3-4B-Base}, and \texttt{Qwen3-8B-Base}, respectively. The gains also span scripts and resource levels; CAGrad improves Arabic (\texttt{ar}) by $+2.8$ points on \texttt{Qwen3-8B-Base}, while PCGrad improves Swahili (\texttt{sw}) by $+1.7$ points on \texttt{Qwen3-4B-Base}. PCGrad improves all eight languages on \texttt{Qwen3-4B-Base}, and CAGrad does so on \texttt{Qwen3-8B-Base}. Overall, the gains are broadly distributed, but no single gradient-combination strategy dominates for every backbone and language.

\paragraph{Generalization to Unseen Languages.}
Table \ref{tab:per_language_unseen} shows similarly broad zero-shot transfer. Across the four unseen European languages (\texttt{de, es, fr, pt}), Bucket-Level MOO improves 46 of 48 comparisons with balanced SFT. Illustrative gains include $+3.8$ points in Spanish (\texttt{es}) from PCGrad on \texttt{Meta-Llama-3-8B}, $+3.5$ points in Portuguese (\texttt{pt}) from MGDA on \texttt{Llama-3.1-8B}, and $+2.8$ points in German (\texttt{de}) from CAGrad on \texttt{Qwen3-4B-Base}. These results indicate that resolving conflicts among the eight supervised languages can also benefit related languages without direct training data.

Japanese (\texttt{ja}) is more architecture dependent. Relative to balanced SFT, the MOO variants change performance by $-1.1$ to $-0.2$ points on \texttt{Meta-Llama-3-8B} and by $-0.8$ to $+0.8$ points on \texttt{Llama-3.1-8B}. In contrast, every MOO variant improves Japanese on both Qwen models: the gains range from $+1.8$ to $+3.9$ points on \texttt{Qwen3-4B-Base} and from $+0.4$ to $+1.0$ points on \texttt{Qwen3-8B-Base}. This difference suggests that zero-shot transfer to a typologically and orthographically distant language depends on both conflict resolution and the multilingual capacity of the pretrained backbone.

\paragraph{Mitigating Language-Specific Forgetting.}
The per-language results also show that multilingual fine-tuning can overwrite capabilities already present in stronger base models. Across the two Qwen backbones, Vanilla SFT reduces 22 of the 26 seen- and unseen-language scores relative to the corresponding pretrained models; balanced loading reduces this count only to 19. Bucket-Level MOO recovers many of these losses. On \texttt{Qwen3-4B-Base}, Bengali (\texttt{bn}) drops from 55.7 to 52.5 under Vanilla SFT but reaches 56.5 with CAGrad, surpassing the pretrained score. Japanese decreases from 67.2 to 62.7 under Vanilla SFT and recovers to 67.1 with MGDA. On \texttt{Qwen3-8B-Base}, MGDA raises Korean (\texttt{ko}) from 69.4 to 72.6, above the pretrained score of 70.7, while CAGrad raises Arabic (\texttt{ar}) from 68.5 to 72.2. These cases suggest that Bucket-Level MOO supports multilingual adaptation both by learning the fine-tuning tasks and by preserving language-specific capabilities that standard SFT can overwrite.

\section{Multilingual Representation Analysis}
\label{app:represenation}

\begin{figure*}[t!]
    \centering
    \includegraphics[width=0.24\linewidth]{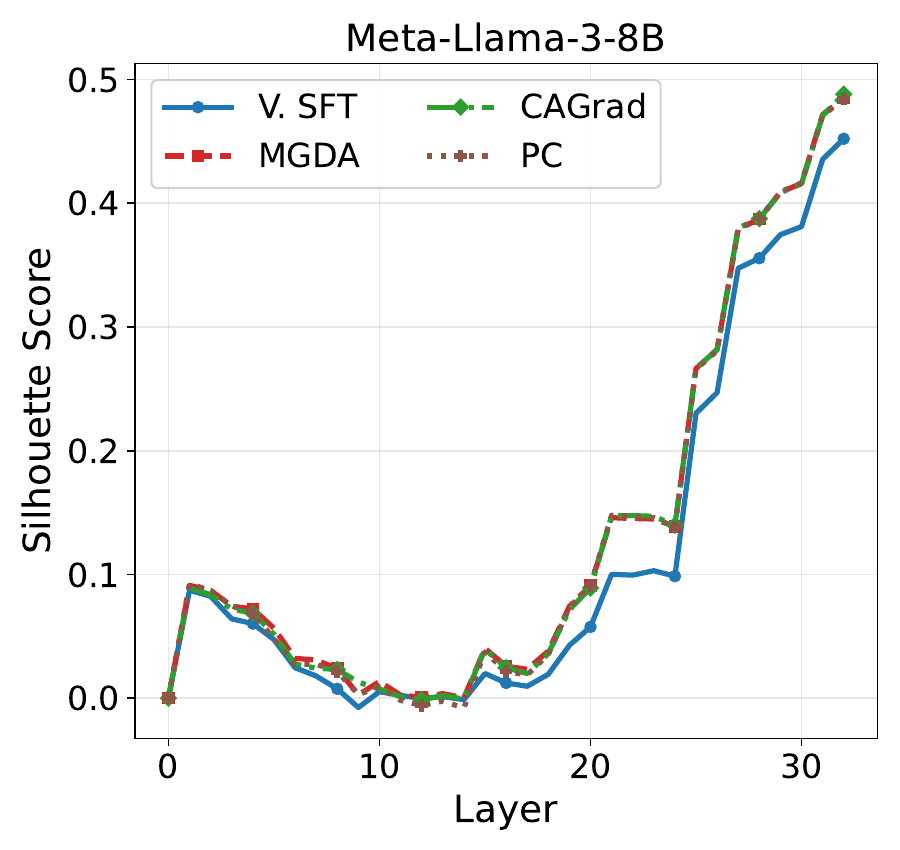}
    \includegraphics[width=0.24\linewidth]{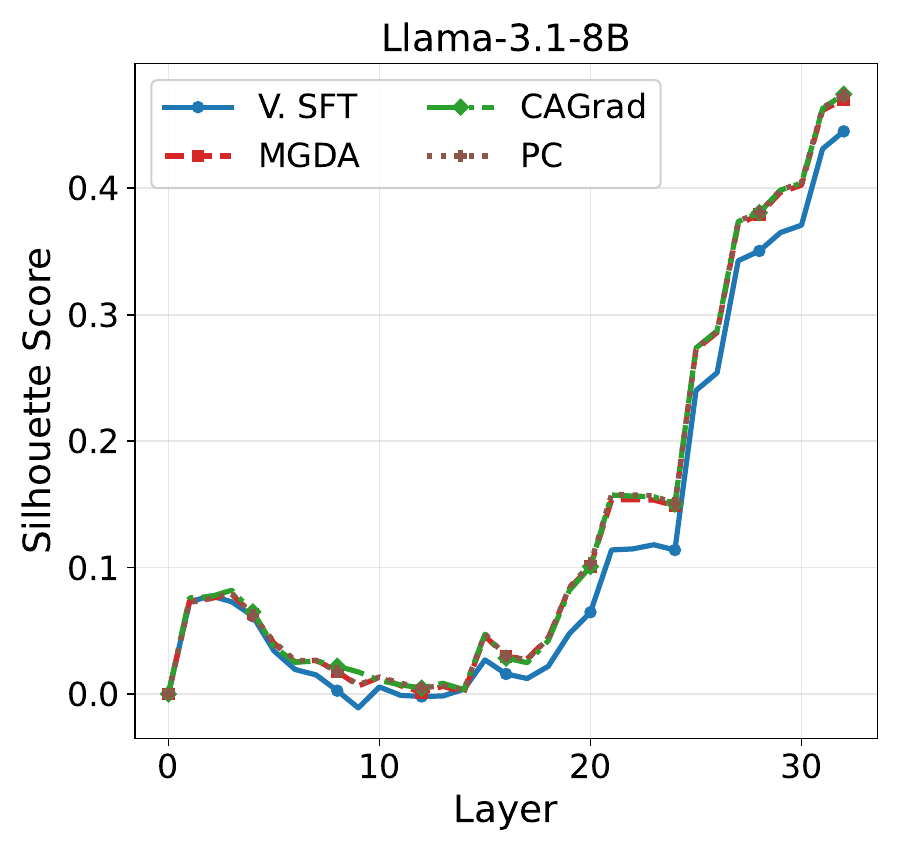}
    \includegraphics[width=0.24\linewidth]{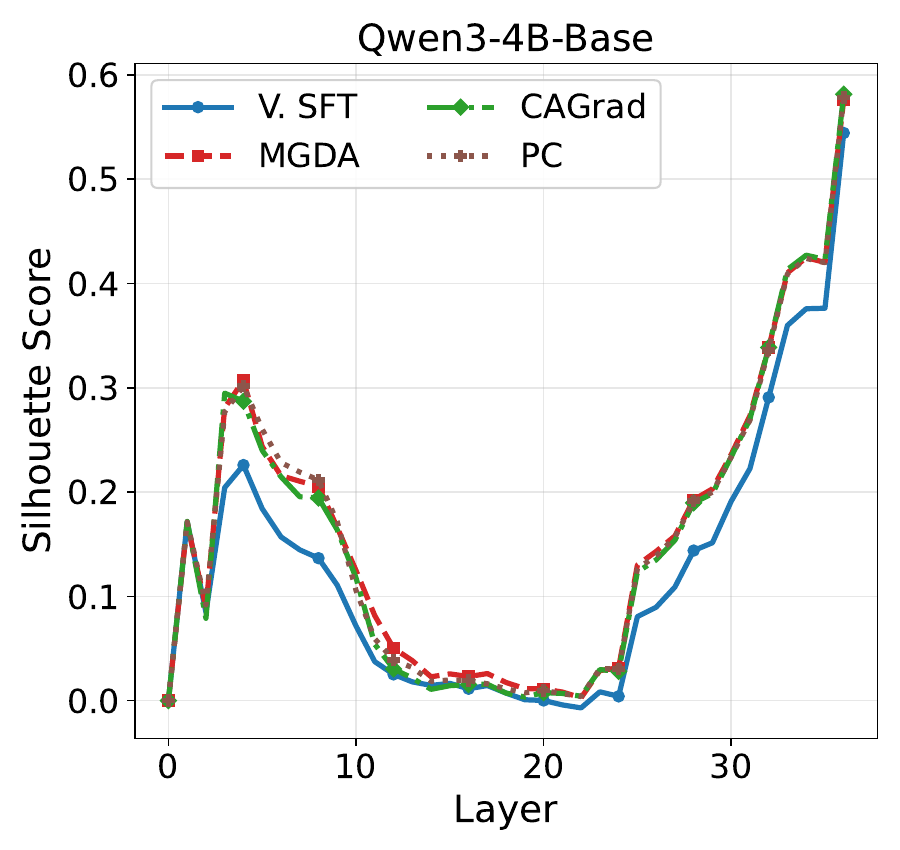}
    \includegraphics[width=0.24\linewidth]{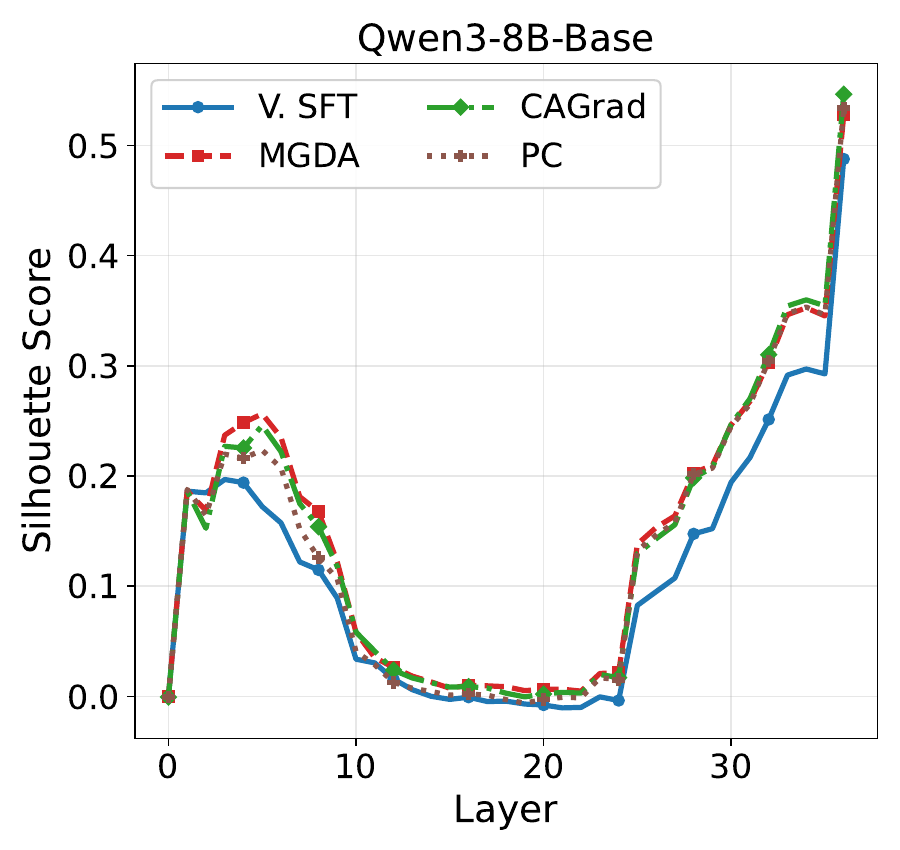}
    \caption{Layer-wise language-clustering Silhouette scores for all four backbones after LIMA--s1 fine-tuning.}
    \label{fig:representation_appendix}
\end{figure*}

We examine whether Bucket-Level MOO changes how multilingual information is organized across model depth. After LIMA--s1 fine-tuning, we extract the final-input-token hidden state at every layer for all 13,040 training samples (1,630 per seen language), and compute the Silhouette score \citep{ROUSSEEUW198753} using language identity as the cluster label. Higher values indicate representations that are more compact within languages and more distinct across languages. The comparison is against Vanilla SFT; the balanced-loader control is not included here.

Figure \ref{fig:representation_appendix} reveals a depth-dependent effect that is consistent across all four backbones. All Bucket-Level variants improve early-layer separability over Vanilla SFT. Scores then approach zero in the middle layers for every method, suggesting that those layers still mix languages within a shared space. Separation re-emerges in the network's final third, where the Bucket-Level variants lead across several consecutive layers; on \texttt{Llama-3.1-8B}, for example, the advantage holds throughout approximately the final ten layers. Although every method spikes at the output-facing layer, the earlier late-layer gap shows that the effect extends beyond the final prediction interface.

The absolute scale varies by backbone: the Qwen3 models reach peak scores near $0.60$, compared with approximately $0.45$ for the Llama models. That the improvement appears at both levels of initial multilingual alignment suggests it is not confined to a particular architecture or degree of pretrained separability.

Greater language separability may appear to sit uneasily with the improved zero-shot transfer reported in Section \ref{sec:main_results}, since cross-lingual generalization is often attributed to shared representations. The depth profile suggests the two are compatible: separation increases at the input- and output-facing stages, where language-specific encoding and decoding take place, while the middle layers remain mixed and continue to carry shared computation. On this reading, localized conflict resolution sharpens language-conditioned structure at the interfaces without partitioning the network's shared processing. Silhouette scores characterize representational geometry, however, and we treat this evidence as an association rather than a causal explanation of the downstream gains.

\end{document}